\documentclass[nohyperref]{article}

\usepackage{microtype}
\usepackage{graphicx}
\usepackage{subcaption}
\usepackage{booktabs} % for professional tables

\usepackage[colorlinks=true, linkcolor=red, urlcolor=blue, citecolor=blue, anchorcolor=blue,backref=page]{hyperref}
\usepackage{amsfonts}
\usepackage{nicefrac}
\usepackage{amsmath,amsthm,amssymb}
\usepackage{fancybox}
\usepackage{float}
\usepackage{color}
\usepackage{tikz}
\usepackage{array}
\usepackage{enumitem}
\usepackage{wrapfig}
\usepackage{etoolbox}
\usepackage{diagbox}
\usepackage{comment}
\usepackage{makecell}
\usepackage{multirow}
\usepackage{bbm}
\usepackage{xcolor,soul,colortbl}

\usepackage[framemethod=TikZ]{mdframed}
\mdfsetup{skipabove=6pt,
innertopmargin=-6pt,skipbelow=0pt}

\usepackage{scalerel,stackengine}
\stackMath
\newcommand\reallywidehat[1]{%
\savestack{\tmpbox}{\stretchto{%
  \scaleto{%
    \scalerel*[\widthof{\ensuremath{#1}}]{\kern.1pt\mathchar"0362\kern.1pt}%
    {\rule{0ex}{\textheight}}%WIDTH-LIMITED CIRCUMFLEX
  }{\textheight}% 
}{2.4ex}}%
\stackon[-6.9pt]{#1}{\tmpbox}%
}

\newtheoremstyle{slplain}% name
  {.4\baselineskip\@plus.1\baselineskip\@minus.1\baselineskip}% Space above
  {.3\baselineskip\@plus.1\baselineskip\@minus.1\baselineskip}% Space below
  {\itshape}% Body font
  {}%Indent amount (empty = no indent, \parindent = para indent)
  {\bfseries}%  Thm head font
  {.}%       Punctuation after thm head
  { }%      Space after thm head: " " = normal interword space;
  {}%       Thm head spec
\theoremstyle{slplain} % italics

\newtheorem*{definition*}{Definition}
\newtheorem*{theorem*}{Theorem}
\newtheorem{theorem}{Theorem}[section]
\newtheorem{lemma}[theorem]{Lemma}
\newtheorem{proposition}[theorem]{Proposition}
\newtheorem{remark}[theorem]{Remark}

\newtheorem{definition}[theorem]{Definition}

\makeatletter
\newtheorem*{rep@theorem}{\rep@title}
\newcommand{\newreptheorem}[2]{%
\newenvironment{rep#1}[1]{%
 \def\rep@title{#2 \ref{##1}}%
 \begin{rep@theorem}}%
 {\end{rep@theorem}}}
\makeatother

\newreptheorem{theorem}{Theorem}
\newreptheorem{lemma}{Lemma}
\newreptheorem{claim}{Claim}
\newreptheorem{corollary}{Corollary}
\newreptheorem{proposition}{Proposition}
\newreptheorem{remark}{Remark}

\theoremstyle{definition}

\theoremstyle{plain} % choose from plain/definition/remark

\numberwithin{equation}{section}

\newtheoremstyle{etplain}% name
  {.0\baselineskip\@plus.1\baselineskip\@minus.1\baselineskip}% Space above
  {.0\baselineskip\@plus.1\baselineskip\@minus.1\baselineskip}% Space below
  {\itshape}% Body font
  {}%Indent amount (empty = no indent, \parindent = para indent)
  {\bfseries}%  Thm head font
  {.}%       Punctuation after thm head
  { }%      Space after thm head: " " = normal interword space;
  {}%       Thm head spec

\newcommand{\R}{\mathbb{R}}
\newcommand{\norm}[1]{\left|\left| #1 \right|\right|}
\newcommand{\Prob}{\mathbb{P}}

\renewcommand\bar\overline

\DeclareMathOperator*{\esssup}{ess\,sup}
\DeclareMathOperator*{\rhosup}{\rho\text{-}ess\,sup}
\DeclareMathOperator{\supp}{supp}

\newcolumntype{C}[1]{>{\centering\let\newline\\\arraybackslash\hspace{0pt}}m{#1}}

\definecolor{Gray}{gray}{0.9}

\DeclareMathOperator{\E}{\mathbb{E}}

\DeclareMathOperator{\CVaR}{CVaR}
\DeclareMathOperator{\SR}{SR}
\DeclareMathOperator{\AR}{AR}
\DeclareMathOperator{\PR}{PR}

\newcommand{\abs}[1]{\ensuremath{{\left\vert #1 \right\vert}}}

\DeclareMathOperator{\sign}{sign}
\DeclareMathOperator{\indicator}{\mathbb{I}}
\newcommand{\ceil}[1]{\left \lceil #1 \right \rceil}

\newcommand{\calA}{\ensuremath{\mathcal{A}}}
\newcommand{\calB}{\ensuremath{\mathcal{B}}}

\newcommand{\calH}{\ensuremath{\mathcal{H}}}

\newcommand{\calN}{\ensuremath{\mathcal{N}}}

\newcommand{\calS}{\ensuremath{\mathcal{S}}}

\newcommand{\calX}{\ensuremath{\mathcal{X}}}
\newcommand{\calY}{\ensuremath{\mathcal{Y}}}

\newcommand{\bbR}{\ensuremath{\mathbb{R}}}

\newcommand{\fkD}{\ensuremath{\mathfrak{D}}}

\newcommand{\fkp}{\ensuremath{\mathfrak{p}}}

\newcommand{\fkr}{\ensuremath{\mathfrak{r}}}

\def\nd/{\textsuperscript{nd}}
\def\rd/{\textsuperscript{rd}}
\def\th/{\textsuperscript{th}}

\newcommand{\setR}{\bbR}

\def\nnil{\nil}
\newcounter{prob}
\newenvironment{prob}[1][\nil]{%
	\def\tmp{#1}
	\equation
	\ifx\tmp\nnil
		\refstepcounter{prob}
		\tag{P\Roman{prob}}
	\else
		\tag{\tmp}
	\fi
	\aligned%
}{%
	\endaligned\endequation%
}

\newenvironment{prob*}{%
	\csname equation*\endcsname%
	\aligned%
}{%
	\endaligned%
	\csname endequation*\endcsname%
}

\usepackage[accepted]{icml2022}

\usepackage[capitalize,noabbrev]{cleveref}

\icmltitlerunning{Probabilistically Robust Learning}

\begin{document}

\twocolumn[
\icmltitle{Probabilistically Robust Learning:\\ Balancing Average- and Worst-case Performance}




\begin{icmlauthorlist}
\icmlauthor{Alexander Robey}{penn}
\icmlauthor{Luiz F.\ O.\ Chamon}{cal}
\icmlauthor{George J.\ Pappas}{penn}
\icmlauthor{Hamed Hassani}{penn}

\end{icmlauthorlist}

\icmlaffiliation{penn}{Department of Electrical and Systems Engineering, University of Pennsylvania, Philadelphia, PA, USA}
\icmlaffiliation{cal}{University of California, Berkeley, Berkeley, CA, USA}

\icmlcorrespondingauthor{Alexander Robey}{arobey1@seas.upenn.edu}

\icmlkeywords{Machine Learning, ICML}

\vskip 0.3in
]



\printAffiliationsAndNotice{}  % leave blank if no need to mention equal contribution

\begin{abstract}
Many of the successes of machine learning are based on minimizing an averaged loss function. However, it is well-known that this paradigm suffers from robustness issues that hinder its applicability in safety-critical domains. These issues are often addressed by training against worst-case perturbations of data, a technique known as adversarial training. Although empirically effective, adversarial training can be overly conservative, leading to unfavorable trade-offs between nominal performance and robustness.  To this end, in this paper we propose a framework called \emph{probabilistic robustness} that bridges the gap between the accurate, yet brittle average case and the robust, yet conservative worst case by enforcing robustness to most rather than to all perturbations. From a theoretical point of view, this framework overcomes the trade-offs between the performance and the sample-complexity of worst-case and average-case learning.  From a practical point of view, we propose a novel algorithm based on risk-aware optimization that effectively balances average- and worst-case performance at a considerably lower computational cost relative to adversarial training.  Our results on MNIST, CIFAR-10, and SVHN illustrate the advantages of this framework on the spectrum from average- to worst-case robustness.
\end{abstract}
\section{Introduction}

Underlying many of the modern successes of learning is the statistical paradigm of empirical risk minimization~(ERM), in which the goal is to minimize a loss function averaged over data~\cite{vapnik1999nature}. Although ubiquitous in practice, it is now well-known that prediction rules learned by ERM suffer from a severe lack of robustness, which in turn greatly limits their applicability in safety-critical domains~\cite{biggio2013evasion,shen2021improving}. Indeed, this vulnerability has led to a pronounced interest in improving the robustness of modern learning tools~\cite{ goodfellow2014explaining, madry2017towards,zhang2019theoretically}.

\begin{figure*}
    \centering
    \includegraphics[width=0.9\textwidth]{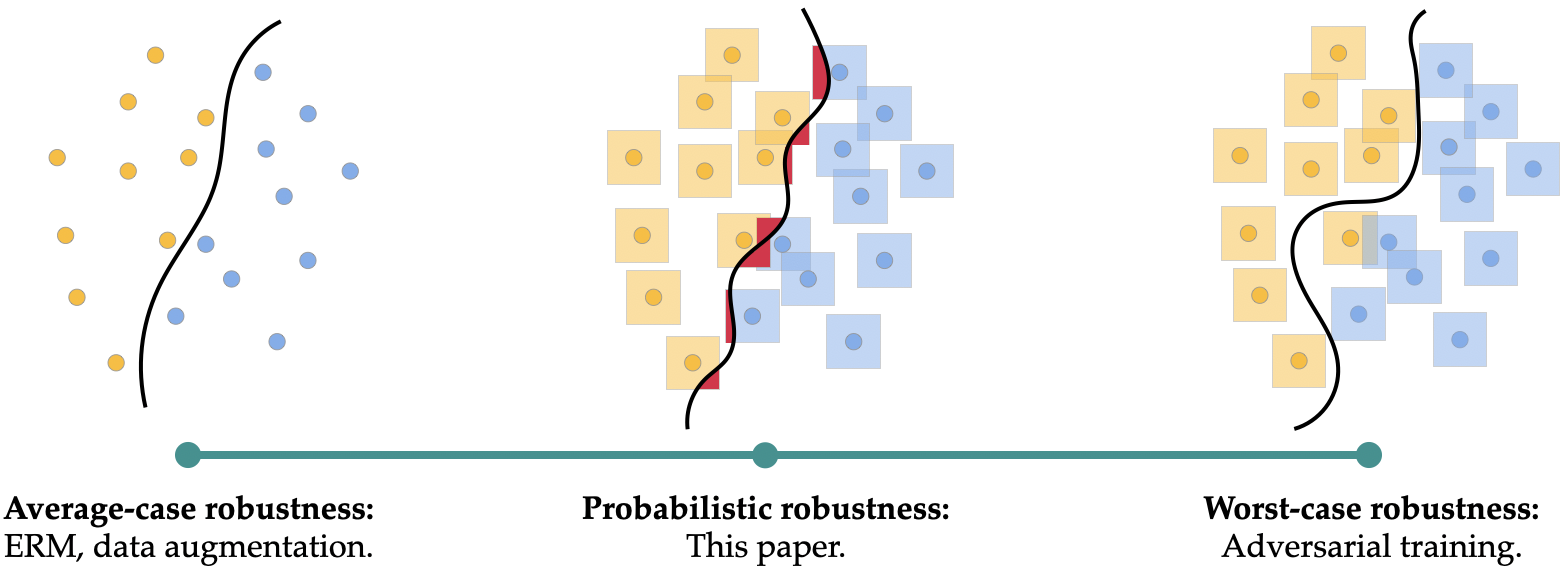}
    \caption{\textbf{The spectrum of robustness.}  Illustration of the different decision boundaries engendered by robustness paradigms. \emph{Left}: the two classes~(yellow and blue dots) can be separated by a simple decision boundary, though it may not be robust to data perturbations. \emph{Right}: the decision boundary must account for the neighborhood of each data point~(yellow and blue boxes), leading to a degraded nominal performance. \emph{Middle}: probabilistic robustness bridges these extremes by allowing a small proportion of perturbations~(shown in red) to be misclassified, mitigating the trade-offs between robustness and accuracy both in theory and in practice.}
    \label{fig:spectrum-of-robustness}
\end{figure*}

To this end, a growing body of work has motivated a learning paradigm known as \emph{adversarial training}, wherein rather than training on the raw data, predictors are trained against worst-case perturbations of data~\cite{goodfellow2014explaining, madry2017towards}. Yet, despite  ample empirical evidence showing that adversarial training improves the robustness of learned predictors~\cite{su2018robustness,croce2020robustbench,tang2021robustart}, this framework is not without drawbacks. Indeed, adversarial training is known to be overly conservative~\cite{tsipras2018robustness,raghunathan2019adversarial}, a property often exhibited by other worst-case approaches ranging from complexity theory~\cite{spielman2004smoothed} to robust control~\cite{zhou1998essentials}. Furthermore, there are broad classes of problems for which the sample complexity of learning a robust predictor is arbitrarily large~\citep{cullina2018pac, montasser2019vc}. Finally, the problem of computing worst-case perturbations of data is nonconvex and underparameterized for most modern learning models including deep neural networks (DNNs).

% \begin{enumerate}[nolistsep]
%     \item[(i)] \textbf{Interpolation.}  The framework should strike a balance between nominal and adversarial performance.
%     \item [(ii)] \textbf{Interpretability.}  This interpolation should be precisely controlled by an interpretable parameter.
%     \item[(iii)] \textbf{Tractability.}  The framework should admit an algorithm which is computationally tractable.
% \end{enumerate}

The fundamental drawbacks of these learning paradigms motivate the need for a new robust learning framework that (i)~avoids the conservatism of adversarial robustness without incurring the brittleness of ERM, (ii)~provides an interpretable way to balance nominal performance and robustness, and (iii)~admits an efficient and effective algorithm. To this end, in this paper we propose a framework called \emph{probabilistic robustness} that bridges the gap between the accurate, yet brittle average-case approach of ERM and the robust, yet conservative worst-case approach of adversarial training. By enforcing robustness to most rather than to all perturbations, we show theoretically and empirically that probabilistic robustness meets the desiderata in~(i)--(iii).  Indeed, our approach parallels a litany of past work in a variety of fields, including smoothed analysis~\citep{spielman2004smoothed} and control theory~\cite{campi2008exact}, wherein robustness is enforced with high probability rather than in the worst case.  In particular, our contributions include:
\begin{itemize}[nolistsep,leftmargin=1em]
    \item \textbf{Novel robustness framework.}  We introduce \emph{probabilistically robust learning}, a new formulation wherein the goal is to learn predictors that are robust to most rather than to all perturbations~(see Fig.~\ref{fig:spectrum-of-robustness}).
    
    \item \textbf{(Lack of) Fundamental trade-offs.} We show that in high dimensional settings, the nominal performance of probabilistically robust classifiers is the same as the Bayes optimal classifier, which contrasts with analogous results for adversarially robust classifiers.
    
    \item \textbf{Sample complexity.} We also show that while the sample complexity of adversarial learning can be arbitrarily high, the sample complexity of our probabilistically robust learning is the same as ERM.
    
    \item \textbf{Tractable algorithm.} Inspired by risk-aware optimization, we propose a tractable algorithm for probabilistically robust training that spans the full spectrum of robustness~(Fig.~\ref{fig:spectrum-of-robustness}) at a considerably lower computational cost than adversarial training.
    
    \item \textbf{Thorough experiments.}  We provide thorough experiments on MNIST, CIFAR-10, and SVHN.  In particular, when we evaluate the ability of algorithms to be robust to 99\% of points in $\ell_\infty$-balls on CIFAR-10, our algorithm outperforms all baselines by six percentage points.
\end{itemize}

\vspace{-0.8em}

\section{Adversarially Robust Learning}

% \subsection{Learning and ERM}

In this paper, we consider the standard supervised learning setting in which data is distributed according to an unknown joint distribution~$\fkD$ over instance-label pairs~$(x,y)$, with instances~$x$ drawn from~$\calX \subseteq \setR^d$ and labels~$y$ drawn from $\calY \subseteq \setR$; in particular, for classification problems we let~$\calY = \{1,\dots,K\}$.  Our goal is to obtain a hypothesis~$h: \calX \to \calY$ belonging to a given hypothesis class $\calH$ that correctly predicts the corresponding label $y$ for each instance $x$.  One common approach to this problem is to minimize a suitably-chosen loss function ~$\ell:\calY \times \calY \to \R_{+}$~(e.g., the 0-1, cross-entropy, or squared loss) on average over~$\fkD$. Explicitly,
\begin{prob}[\textup{P-NOM}]\label{eq:nom-training}
	\min_{h \in \calH}\ \SR(h) \triangleq \E_{(x,y)\sim\fkD}\!\Big[ \ell\big( h(x), y \big) \Big]
		\text{.}
\end{prob}
Here~$\SR(h)$ denotes the standard risk or \emph{nominal performance} of~$h$.\footnote{We assume that~$\ell$ and~$h$ satisfy the integrability conditions needed for the expectation in~\eqref{eq:nom-training} to be well-defined.} The hypothesis class~$\calH$ is often comprised of models~$f_{\theta}$ parameterized by a vector~$\theta$ drawn from a compact set~$\Theta \subset \R^p$, e.g., linear classifiers or deep neural networks with bounded parameters.

Because the distribution~$\fkD$ is unknown, the objective in~\eqref{eq:nom-training} cannot be evaluated in practice.  The core idea behind ERM is to use samples~$(x_j,y_j)$ drawn i.i.d.\ from~$\fkD$ to estimate the expectation:
\begin{prob}[\textup{P-ERM}]\label{eq:erm}
	\min_{h \in \calH}\ \frac{1}{N} \sum_{j = 1}^N \ell(h(x_j), y_j)
		\text{.}
\end{prob}
One of the fundamental problems in learning theory is to establish the number~$N$ of i.i.d.\ samples needed for~\eqref{eq:erm} to approximate the value of~\eqref{eq:nom-training} with high probability.  Problems for which~$N$ is finite are called probably approximately correct~(PAC) learnable~\cite{vapnik1999nature}.

\textbf{Pitfalls of ERM.}  While solving~\eqref{eq:erm} often yields classifiers that are near-optimal for~\eqref{eq:nom-training}, there is now overwhelming evidence that these hypotheses are sensitive to imperceptible perturbations of their input~\cite{biggio2013evasion,szegedy2013intriguing}. Explicitly, given an instance~$x$ and a solution~$h^\star$ for~\eqref{eq:erm}, one can often find a small perturbations~$\delta$ such that~$h(x+\delta) \neq h(x) = y$.\footnote{For conciseness, we focus on perturbations of the form~$x \mapsto x+\delta$. However, our results also apply to more general models, such as those in~\cite{robey2020model,wong2020learning}.}  This issue has been observed in hypotheses ranging from linear models to complex nonlinear models~(e.g., DNNs) and has motivated a considerable body of recent work on robust learning~\cite{goodfellow2014explaining, madry2017towards,jalal2017robust,zhang2019theoretically,kamalaruban2020robust,rebuffi2021fixing}.

\subsection{Adversarial robustness}

Among the approaches that have been proposed to mitigate the sensitivity of hypotheses to input perturbations, there is considerable empirical evidence suggesting that adversarial training is an effective way to obtain adversarially robust classifiers~\cite{su2018robustness,athalye2018obfuscated,croce2020robustbench}. In this paradigm, hypotheses are trained against worst-case perturbations of data rather than on the raw data itself, giving rise to a robust counterpart of~\eqref{eq:nom-training}:
\begin{prob}[\textup{P-ROB}]\label{eq:p-rob}
    \min_{h \in \calH}\ \AR(h) \triangleq \E_{(x,y)} \! \left[
    	\sup_{\delta\in\Delta}\ \ell\big( h(x+\delta), y \big)
    \right] \text{,}
\end{prob}
where~$\Delta \subset \setR^d$ is the set of allowable perturbations and we omit the distribution~$\fkD$ for simplicity. In~\eqref{eq:p-rob},~$\AR(h)$ denotes the \emph{adversarial risk} of~$h$. Observe that in contrast to~\eqref{eq:nom-training}, the objective of~\eqref{eq:p-rob} explicitly penalizes hypotheses that are sensitive to perturbations in~$\Delta$, thus yielding more robust hypotheses. Numerous principled adversarial training algorithms have been proposed for solving~\eqref{eq:p-rob}~\cite{goodfellow2014explaining, madry2017towards, kannan2018adversarial} and closely-related variants~\cite{moosavi2016deepfool, wong2018provable, wang2019improving, zhang2019theoretically}.  

\textbf{Pitfalls of adversarial training.}  Despite the empirical success of adversarial training at defending against worst-case attacks, this paradigm has several limitations. In particular, it is well-known that the improved adversarial robustness offered by~\eqref{eq:p-rob} comes at the cost of degraded nominal performance~\cite{tsipras2018robustness, dobriban2020provable, javanmard2020precise,yang2020closer}. Additionally, evaluating the supremum in~\eqref{eq:p-rob} can be challenging in practice, since the resulting optimization problem is nonconcave and underparameterized for modern hypothesis classes, e.g., DNNs~\cite{soltanolkotabi2018theoretical}.  Finally, from a learning theoretic perspective, there exist hypothesis classes for which~\eqref{eq:nom-training} is PAC learnable while~\eqref{eq:p-rob} is not, i.e., for which~\eqref{eq:nom-training} can be approximated using samples whereas~\eqref{eq:p-rob} cannot~\cite{cullina2018pac, montasser2019vc,diochnos2019lower}.

\subsection{Between the average and worst case}

Aside from the now prevalent framework of adversarial training, many works have proposed alternative methods to mitigate the aforementioned vulnerabilities of learning. A standard technique that dates back to~\cite{holmstrom1992using} is to use a form of data augmentation~\eqref{eq:nom-training}:
\begin{prob}[\textup{P-AVG}]\label{eq:p-avg}
    \min_{h \in \calH} \: \E_{(x,y)} \! \Big[ \E_{\delta \sim \fkr}\!
    	\big[ \ell(h(x+\delta),y) \big]
    \Big]
    	\text{.}
\end{prob}
Here, the inner expectation is taken against a known distribution~$\fkr$. While many algorithms have been proposed for specific~$\fkr$~\cite{krizhevsky2012imagenet, hendrycks2019augmix,laidlaw2019functional,chen2020group}, they fail to yield classifiers sufficiently robust to small perturbations.

Toward obtaining robust alternatives to~\eqref{eq:p-avg}, two recent works propose relaxations of~\eqref{eq:p-rob} that engender notions of robustness between~\eqref{eq:nom-training} and~\eqref{eq:p-rob}. The first relies on the hierarchy of Lebesgue spaces, i.e.,
\begin{prob}\label{eq:rice}
    \min_{h_q \in \calH}\:\mathbb{E}_{(x,y)} \Big[
	    \big\Vert \ell(h_q(x+\delta),y) \big\Vert_{L^q}
    \Big] \text{,}
\end{prob}
where~$\norm{\cdot}_{L^q}$ denotes the Lebesgue~$q$-norm taken over $\Delta$ with respect to the measure~$\fkr$~\cite{rice2021robustness}. The second relaxes the supremum using the soft maximum or LogSumExp function~\cite{li2020tilted, li2021tilted}:
\begin{prob}\label{eq:term}
    \min_{h_t \in \calH} \: \E_{(x,y)} \left[
    	\frac{1}{t} \log\left(
    		\mathbb{E}_{\delta \sim \fkr} \left[ e^{ t \cdot \ell(h_t(x+\delta),y)}  \right]
	    \right)
    \right]
        \text{.}
\end{prob}
While both~\eqref{eq:rice} and~\eqref{eq:term} are strong alternatives to~\eqref{eq:p-rob}, both suffer from significant practical issues related to optimizing their objectives.  More specifically, the objective in~\eqref{eq:rice} cannot be efficiently computed during training due to the difficulty of evaluating the $L^q$ norm~\cite{rice2021robustness}.  And in the case of~\eqref{eq:term}, the gradient of the objective becomes unstable for large values of $t$ when training DNNs.  Furthermore, looking beyond these practical limitations, there is also no clear relationship between the values of~$q$ and~$t$ and robustness properties of the solutions for~\eqref{eq:rice} and~\eqref{eq:term}, making these parameters difficult to choose or interpret. These limitations motivate the need for an alternative formulation of robust learning.

\begin{remark}\label{R:extremes}
Formally, the limiting cases of~\eqref{eq:rice} and~\eqref{eq:term} are not~\eqref{eq:nom-training} and~\eqref{eq:p-rob}. Indeed, for~$q = 1$ and~$t \to 0$, the objectives of both problems approach the objective of~\eqref{eq:p-avg}.  For~$q = \infty$ and~$t \to \infty$, the objectives of~\eqref{eq:rice} and~\eqref{eq:term} can be written in terms of the essential supremum
\begin{prob}\label{eq:ess_rob}
    \min_{h_r \in \calH} \: \E_{(x,y)} \! \left[
    	\esssup_{\delta \sim \fkr}\ \ell\big( h_r(x+\delta), y \big)
    \right]
    	\text{,}
\end{prob}
where~$\esssup_{\delta \sim \fkr} f(\delta)$ denotes an almost everywhere upper bound of~$f$, i.e., an upper bound except perhaps on a set of $\fkr$-measure zero. Note that the essential supremum is a weaker adversary~($\esssup \leq \sup$), although for rich enough hypothesis classes, the value of~\eqref{eq:p-rob} and~\eqref{eq:ess_rob} can be the same~\citep[Lemma 3.8]{bungert2021geometry}.
\end{remark}
\section{Probabilistically robust learning} \label{S:formulation}

The discussion in the previous section identifies three desiderata for a new robust learning framework:
\begin{enumerate}[nolistsep]
    \item[(i)] \textbf{Interpolation.}  The framework should strike a balance between nominal and adversarial performance.
    \item [(ii)] \textbf{Interpretability.}  This interpolation should be precisely controlled by an interpretable parameter.
    \item[(iii)] \textbf{Tractability.}  The framework should admit a computationally tractable training algorithm.
\end{enumerate}
While~\eqref{eq:rice} and~\eqref{eq:term} do achieve~(i), neither meets the criteria in~(ii) or~(iii). On the other hand, the probabilistic robustness framework introduced in this section satisfies all of these desiderata. Moreover, as we show in Section~\ref{sect:analysis}, it benefits from numerous theoretical properties.

\subsection{A probabilistic perspective on robustness} 

The core idea behind probabilistic robustness is to replace the worst-case view of robustness with a probabilistic perspective.  This idea has a long history in numerous fields, including chance-constrained optimization in operations research~\cite{charnes1958cost, miller1965chance} and control theory~\cite{campi2008exact, ben2009robust, ramponi2018consistency, schildbach2014scenario}
and smoothed analysis in algorithmic complexity theory~\cite{spielman2004smoothed}. In each of these domains, probabilistic approaches are founded on the premise that a few rare events are disproportionately responsible for the performance degradation and increased complexity of adversarial solutions. In the context of robust learning, this argument is supported by recent theoretical and empirical observations suggesting that low-dimensional regions of small volume in the data space are responsible for the prevalence of adversarial examples~\cite{gilmer2018adversarial, khoury2018geometry, shamir2021dimpled}. This suggests that because the adversarial training formulation~\eqref{eq:p-rob} does not differentiate between perturbations, it is prone to yielding conservative solutions that overcompensate for rare events.

\begin{figure}
    \centering
    \includegraphics[width=\columnwidth]{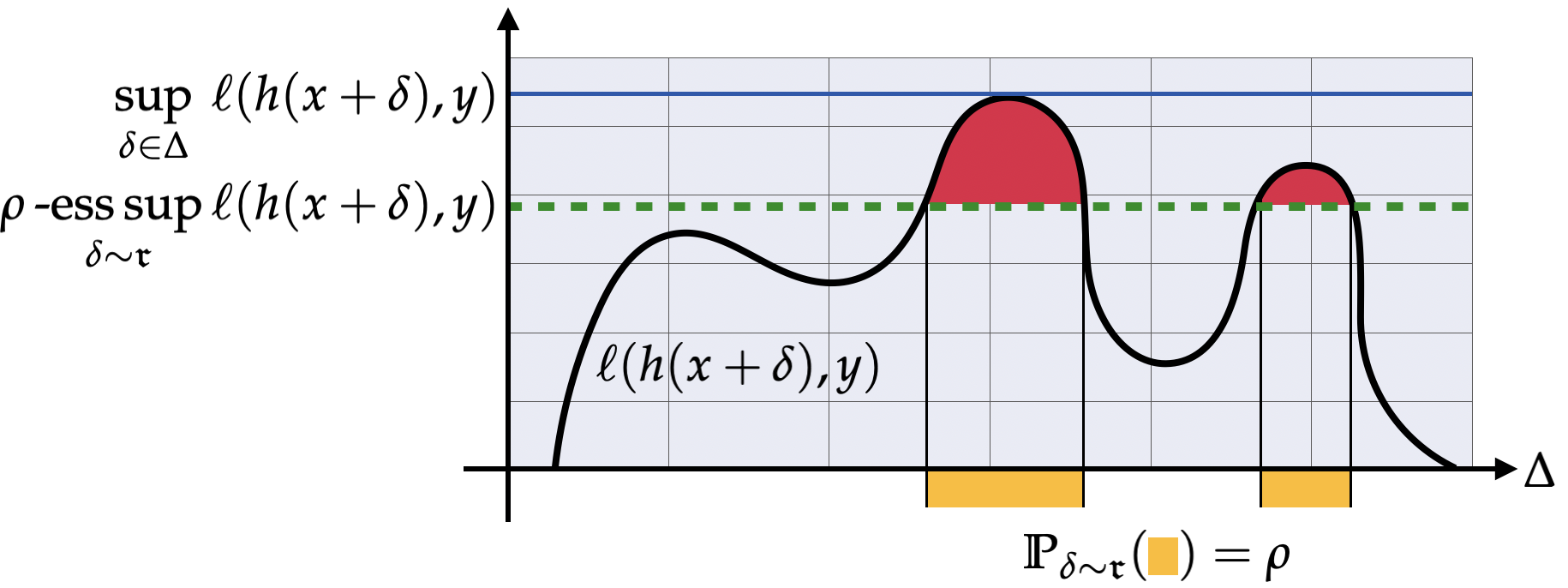}
    \caption{\textbf{The $\pmb{\rhosup}$ operator.}  In this cartoon, we fix $(x,y)\in\Omega$ to show the perturbation set~$\Delta$ on the $x$-axis and the value of $\ell(h(x+\delta),y)$~on the $y$-axis. The solid line shows the value of~$\sup_{\delta\in\Delta}\ell(h(x+\delta),y)$, the least upper bound for~$\ell(h(x+\delta),y)$. The dashed line shows, for a fixed~$\rho>0$, the value of~$\rhosup_{\delta\sim\fkr} \ell(h(x+\delta),y)$, the smallest number~$u$ such that $\ell(h(x+\delta),y)$ takes on values larger than $u$~(shown in red) on a subset~(shown in yellow) with volume not exceeding $\rho$.}
    \label{fig:esssup-cartoon}
\end{figure}

To begin our exposition of probabilistic robustness, first consider the case of the 0-1 loss. Here, adversarial training has shifted focus from the nominal 0-1 loss $\indicator[h(x)\neq y]$ to the adversarially robust 0-1 loss $\indicator[\exists\delta\in\Delta \text{ s.t.\ } h(x+\delta)\neq y]$.  Concretely, this robust loss takes value one if there exists any perturbation $\delta$ in a neighborhood $\Delta$ of a fixed instance $x$ that causes misclassification, and value zero otherwise. Motivated by our discussion in the previous paragraph, we seek a relaxed variant of the robust loss which will allow us to quantify the robust performance of a candidate hypothesis while ignoring regions of insignificant volume in~$\Delta$.  To do so, we first introduce a probability distribution~$\fkr$ (e.g., the uniform distribution) over~$\Delta$ to assess the local probability of error~$\Prob_{\delta\sim\fkr} [h(x+\delta) \neq y]$ around each instance~$x$.  Then, for a fixed tolerance level~$\rho \in [0,1)$, the goal of probabilistic robustness is to minimize the probability that the event $\Prob_{\delta\sim\fkr} [h(x+\delta)= y] < 1- \rho$ will occur; that is, the goal is to ensure that most perturbations do not cause an instance~$x$ to be misclassified.  As such, the smaller the value of~$\rho$, the more stringent the requirement on robustness.  In this way, under the 0-1 loss, our probabilistically robust learning task can then formulated as follows:
\begin{prob}\label{eq:01-loss-unconst}
	\min_{h \in \calH} \: \E_{(x,y)} \Big[\indicator\big[
		\Prob_{\delta \sim \fkr} \left[ h(x+\delta) \neq y \right] > \rho
	\big]\Big]
		\text{.}
\end{prob}
It is then straightforward to see that under the 0-1 loss, probabilistically robust learning is an instance of~\eqref{eq:nom-training} in the sense that we are minimizing the expectation of a particular loss $\indicator[\Prob_{\delta \sim \fkr}\left[h(x+\delta) \neq y\right] > \rho]$.

\subsection{Generalizing to general loss functions}  

With this intuition in mind, we now generalize~\eqref{eq:01-loss-unconst} to arbitrary loss functions. To do so, let~$(\Omega,\calB)$ define a measurable space, where~$\Omega = \calX\times\calY$ and~$\calB$ denotes the Borel~$\sigma$-algebra of~$\Omega$. Observe that for fixed $(x,y)\in\Omega$, the supremum~$t^\star := \sup_{\delta\in\Delta} \ell(h(x+\delta),y)$ from~\eqref{eq:p-rob} can be written in epigraph form as
\begin{equation}\label{eq:epigraph-form}
	t^\star = \min_{t\in\R} \  t \ \ \  \text{s.t.} \ \ \  \ell(h(x+\delta),y)\leq t \quad \forall \delta\in\Delta
		\text{.}
\end{equation}
This formulation makes explicit the fact that the supremum is the least upper bound of~$\ell(h(x+\delta),y)$ (see Fig.~\ref{fig:esssup-cartoon}).

As in the development of~\eqref{eq:01-loss-unconst}, however, we do not need~$t$ to upper bound~$\ell(h(x+\delta),y)$ for all~$\delta \in \Delta$, but only for a proportion~$1-\rho$ of the volume of $\Delta$. We therefore consider the following relaxation of~\eqref{eq:epigraph-form}:
\begin{alignat}{2}\label{eq:stochastic-epigraph}
	u^\star(\rho) = &\min_{u\in\R} \quad  &&u
	\\ 
	&\text{ s.t.} &&  \Prob_{\delta\sim\fkr} \big[ \ell(h(x+\delta),y) \leq u\big] > 1-\rho \notag
\end{alignat}
In contrast to~\eqref{eq:epigraph-form}, the upper bound in~\eqref{eq:stochastic-epigraph} can ignore perturbations for which~$\ell(h(x+\delta),y)$ is large~(red regions in Fig.~\ref{fig:esssup-cartoon}) as long as these perturbations occupy a subset of~$\Delta$ that has probability less than~$\rho$~(yellow regions in Fig.~\ref{fig:esssup-cartoon}). Thus, note that for $\rho \geq \rho^\prime$, it holds that
\begin{equation*}
    u^\star(\rho) \leq u^\star(\rho^\prime) \leq u^\star(0) \leq t^\star
\end{equation*}
and that~$u^\star(0)$ is the essential supremum from measure theory.  In view of this connection, we call~$u^\star(\rho)$ in~\eqref{eq:stochastic-epigraph} the~\emph{$\rho$-essential supremum}~($\rhosup$) (dashed line in Fig.~\ref{fig:esssup-cartoon}) and formalize its definition below.
\begin{definition}\label{D:rho_esssup}
	Let~$(\Omega,\calB,\fkp)$ be a measure space and let $f:\Omega\to\R$ be a measurable function.  Define the set
	\begin{align}
	    U_\rho = \{u \in \setR \mid \fkp ( f^{-1}(u,\infty)) \leq \rho \} \text{ .} \notag
	\end{align}
	Then, the $\rhosup$ is defined as\footnote{While we define~$\rhosup$ as the infimum for~$\rho = 1$, this is done only for consistency as this value will play no significant role in subsequent derivations.}
	\begin{equation*}
		\rhosup_{x \sim \fkp} f(x) =
		\begin{cases}
			\inf \: U_\rho &\rho \in [0,1)
			\\
			\inf\left\{f(x) : x \in \supp(\fkp)\right\} &\rho = 1
		\end{cases}
	\end{equation*}
	where \text{supp}($\fkp$) denotes the support of $\fkp$.
\end{definition}

For a given tolerance level $\rho\in[0,1)$, probabilistically robust learning can now be formalized in full generality as
\begin{mdframed}[roundcorner=5pt, backgroundcolor=yellow!8]
\begin{prob}[P-PRL]\label{eq:p-prl}
	\min_{h_p \in \calH} \: \PR(h_p; \rho) \triangleq \E_{(x,y)} \! \left[
		\rhosup_{\delta \sim \fkr}\ \ell\big( h_p(x+\delta), y \big)
	\right]
\end{prob}
\end{mdframed}
In this problem,~$\fkr$ is defined on the measurable space~$(\Delta,\calB_{\Delta})$, where~$\calB_\Delta$ is the restriction of the $\sigma$-algebra~$\calB$ to~$\Delta$. For consistency, we define the probabilistic robustness problem~\eqref{eq:p-prl} for~$\rho=1$ as~\eqref{eq:p-avg}.

By construction, it is clear that \eqref{eq:p-prl} satisfies desideratum~(i): For $\rho=0$, we recover~\eqref{eq:ess_rob} and for all~$\rho\in(0,1)$, we obtain a strict relaxation of the robustness criteria in~\eqref{eq:p-rob}.  Furthermore, if $\fkr$ is symmetric---meaning that the mean and median coincide, as is the case for the uniform and Gaussian distributions---then we recover~\eqref{eq:p-avg} for $\rho=\nicefrac{1}{2}$.  However, we note that for general distribution $\fkr$ the objective of~\eqref{eq:p-prl} does not approach that of~\eqref{eq:p-avg}.  In practice, this is inconsequential because we are primarily interested in values of~$\rho$ close to zero in order to guarantee robustness in large neighborhoods of the data.  Additionally, as we show in Section~\ref{S:algorithm}, the algorithm we put forward to solve~\eqref{eq:p-prl} yields solutions that exactly recover the average case. Moreover, we show in Sections~\ref{S:algorithm} and~\ref{S:experiments} that this algorithm fulfills desideratum~(iii). 

As for the interpretability of~$\rho$ in desideratum~(ii), notice that the relaxation in~\eqref{eq:p-prl} explicitly minimizes the loss over a neighborhood of~$\fkr$-measure at least~$1-\rho$ of each data point. Thus, in contrast to~\eqref{eq:rice} or~\eqref{eq:term}, this relaxation has a practical interpretation. This interpretability is clearest in the 0-1 loss case~\eqref{eq:01-loss-unconst}, which effectively minimizes~$\Prob_{(x,y)} \big[ \Prob_{\delta \sim \fkr} [h_p(x+\delta) \neq y ] > \rho \big]$.  In this way, probabilistic robustness measures the probability of making an error in a neighborhood of each point and only declares failure if that probability is too large, i.e., larger than~$\rho$. This is in contrast to directly measuring the probability of error as in~\eqref{eq:nom-training} or requiring that the probability of failure vanishes as in~\eqref{eq:ess_rob}.

\section{Statistical properties of probabilistic robustness}
\label{sect:analysis}

In this section, we characterize the behavior of probabilistic robustness in different settings to show that, in addition to meeting the practical desiderata enumerated in Section~\ref{S:formulation}, this framework also enjoys significant statistical advantages over its worst-case counterpart.  In particular, in line with a myriad of past work~\cite{su2018robustness,bhagoji2019lower,dobriban2020provable, javanmard2020precise, cullina2018pac, montasser2020efficiently}, we first observe that the security guarantee of adversarial robustness comes at the cost of degraded nominal performance as well as an arbitrarily large sample complexity. However, in stark contrast to these results, we show that even for arbitrarily small~$\rho$, there exists classes of problems for which probabilistic robustness can be achieved with the same sample complexity as classical learning and at a vanishingly small cost in nominal performance relative to the Bayes optimal classifier.  In the sequel, we first analyze probabilistically robust learning in the two fundamental settings of binary classification and linear regression~(Section~\ref{S:tradeoff}), followed by a learning theoretic characterization of its sample complexity~(Section~\ref{S:sample-comlexity}).

\subsection{Nominal performance vs.\ robustness trade-offs}\label{S:tradeoff}

In this section, we consider perturbation sets of the form~
\begin{align}
    \Delta = \{\delta\in\R^d : \norm{\delta}_2\leq \epsilon\}
\end{align}
for a fixed~$\epsilon > 0$ and we let~$\fkr$ be the uniform distribution over $\Delta$.  We consider binary classification problems with data distributed as
%
% \vspace{-0.3em}
\begin{align} \label{eq:mix-of-gaussian}
    x \mid y \sim \mathcal{N}(y\mu, \sigma^2 I_d), \quad y = \begin{cases}
        +1 &\:\text{w.p. } \pi \\
        -1 &\:\text{w.p. } 1-\pi
    \end{cases}
    \text{,}
\end{align}
where~$\pi\in[0,1]$ is the proportion of the~$y=+1$ class, $I_d$ is the $d$-dimensional identity matrix, and~$\sigma > 0$ is the within-class standard deviation. We assume without loss of generality that the class means~$\pm\mu$ are centered about the origin and, by scaling, that~$\sigma=1$. In this setting, it is well-known that the Bayes optimal classifier is
\begin{align}
    h^\star_\textup{Bayes}(x) = \sign\left(x^\top \mu - q/2\right)
\end{align}
where $q = \ln[ (1-\pi)/\pi]$~\cite{anderson1958introduction}.  Moreover, \cite{dobriban2020provable} recently showed that the optimal adversarially robust classifier is
\begin{align}
    h^\star_r(x) = \sign\left( x^\top\mu \big[ 1 - \epsilon/\norm{\mu}_2 \right]_+ - q/2 \big) \text{ ,}
\end{align}
where~$[z]_+ = \max(0,z)$.  In the following proposition, we obtain a closed-form expression for the optimal probabilistically robust linear classifier.
\begin{proposition}\label{prop:opt-beta-rob-gaussian}
Suppose the data is distributed according to~\eqref{eq:mix-of-gaussian} and let $\epsilon < \norm{\mu}_2$. Then, for~$\rho \in [0, \nicefrac{1}{2}]$,
\begin{align}
    h_p^\star(x) =
        \sign\left(x^\top \mu\left(1 - \frac{v_\rho}{\norm{\mu}_2} \right)_+ - \frac{q}{2}\right) \label{eq:mix-gaussians-opt-classifier}
\end{align}
is the optimal linear solution for~\eqref{eq:01-loss-unconst}, where~$v_\rho$ is the Euclidean distance from the origin to a spherical cap of~$\Delta$ with measure~$\rho$. Moreover, it holds that
\begin{align}
    \PR(h_p^\star; \rho) - \SR(h^\star_\textup{Bayes}) = \begin{cases}
        O\! \left( \frac{1}{\sqrt{d}} \right) \text{,} \!\!\!\! & \rho \in \left(0,\frac{1}{2} \right]
        \\
        O(1) \text{,} & \rho = 0.
    \end{cases} \label{eq:mix-of-gauss-risk-diff}
\end{align}
\end{proposition}

Concretely, Prop.~\ref{prop:opt-beta-rob-gaussian} conveys three messages. 

Firstly, \eqref{eq:mix-gaussians-opt-classifier} shows that the optimal probabilistically robust linear classifier corresponds to the Bayes classifier with an effective mean $\mu\mapsto \mu(1-v_\rho/\norm{\mu}_2)_+$.  Secondly, $h^\star_p$ depends on the tolerance level~$\rho$ through the measure of a spherical cap of~$\Delta$.  Indeed, it is straightforward to check that~$v_{1/2} = 0$ and~$v_0 = \epsilon,$ and thus~\eqref{eq:mix-gaussians-opt-classifier} recovers~$h^\star_\text{Bayes}$ and~$h^\star_r$ respectively. Thus, in this setting, not only does~\eqref{eq:p-prl} interpolate between~\eqref{eq:p-rob} and~\eqref{eq:nom-training} as~$\rho$ varies from~$0$ to~$\nicefrac{1}{2}$, but so too do its optimal solutions.  

Finally, \eqref{eq:mix-of-gauss-risk-diff} shows that the best achievable probabilistic robustness is essentially the same as the best achievable nominal performance in high dimensions, regardless of the value of~$\rho$ provided that it remains strictly positive.  However, in the adversarially robust setting of~$\rho=0$, the gap between robustness and accuracy does not vanish, which lays bare the conservatism engendered by forcing classifiers to account for a small set of rare events. In this way, a phase transition occurs at~$\rho=0$ in the sense that for any $\rho>0$, the gap between nominal performance and probabilistic robustness vanishes in high dimensions, despite the fact that we protect against an arbitrary proportion $1-\rho$ of perturbations. 

In Appendix~\ref{app:lin-regression}, we study the distinct yet related problem of linear regression with Gaussian features.  In this setting, we observe exactly the same phenomenon, wherein the trade-off between nominal performance and probabilistic robustness vanishes in high dimensions.

\subsection{Sample complexity of probabilistic robustness}\label{S:sample-comlexity}

From a learning theoretic perspective, the behavior of adversarial learning is considerably different from that of classical learning. Indeed, the sample complexity of adversarial learning, i.e., the number of samples needed for the empirical counterpart of~\eqref{eq:p-rob} to approximate its solution with high probability, can often be arbitrarily large relative to~\eqref{eq:nom-training}~\cite{cullina2018pac, yin2019rademacher, montasser2019vc}. The following proposition shows that unlike in the case of adversarial robustness, the sample complexity of probabilistically robust learning can match that of classical learning even for arbitrarily small~$\rho>0$.

\begin{proposition} \label{T:sample_complexity}
Let $\ell$ be the 0-1 loss and let~$\fkr$ be fully supported on~$\Delta$ and absolutely continuous with respect to the Lebesgue measure. For any constant~$\rho_o \in (0,\nicefrac{1}{2})$, there exists a hypothesis class~$\calH_o$ such that the sample complexity of probabilistically robust learning at level $\rho$ is
\begin{equation*}
	N = \begin{cases}
		\Theta\big( \log_2(1/\rho_o) \big) \text{,} &\rho = 0
		\\
		\Theta(1) \text{,} &\rho_o \leq \rho \leq 1-\rho_o
	\end{cases}
\end{equation*}
In particular, $\Theta(1)$ is the sample complexity of~\eqref{eq:nom-training}.
\end{proposition}

A formal statement of this result and the requisite preliminaries are provided in~Appendix~\ref{A:learning-theory-proofs}. Concretely, Proposition~\ref{T:sample_complexity} shows that there exist learning problems for which the sample complexity of ERM and PRL are the same, and for which the sample complexity of adversarial training is much larger\footnote{In fact, in the case of~$\rho = 0$, i.e., adversarial robustness, \cite{montasser2019vc} shows that the problem can be unlearnable, i.e., have infinite sample complexity.} than PRL and ERM.  However, the result also highlights the fact that when protecting against an overwhelmingly large proportion $1-\rho_o$ of perturbations, PRL can transition from having the sample complexity of classical learning to that of adversarial learning depending on the value of $\rho$.  Note that the hypothesis class $\calH_o$ of the problem depends on $\rho$, meaning that although the sample complexity can depend on $\rho$, there are still problems for which the sample complexity of PRL is exponentially smaller than adversarial training.  This implies that the conservatism of adversarial learning can manifest itself not only in the form of nominal performance degradation~\cite{tsipras2018robustness}, but also in terms of learning complexity.

% Indeed, given that we can only solve the learning problem~\eqref{eq:p-prl} using samples, this result shows that 

\section{A tractable, risk-aware algorithm} \label{S:algorithm}

So far, we have established that probabilistically robust learning has numerous desirable practical and theoretical properties.  However, the stochastic, non-convex, non-smooth nature of the $\rhosup$ means that in practice solving~\eqref{eq:p-prl} presents a significant challenge. Nevertheless, in this section we show that the~$\rhosup$ admits a tight convex upper bound that can be efficiently optimized using stochastic gradient methods.  Given this insight, we propose a novel algorithm for probabilistically robust learning which is guaranteed to interpolate between~\eqref{eq:p-avg} and~\eqref{eq:p-rob}.

\begin{algorithm}[t]
  \caption{Probabilistically Robust Learning (PRL)}
  \label{alg:cvar-sgd}
\begin{algorithmic}[1]
  \STATE {\bfseries Hyperparameters: } sample size $M$, step sizes~$\eta_\alpha, \eta > 0$, robustness parameter~$\rho > 0$, neighborhood distribution~$\fkr$, num.\ of inner optimization steps~$T$, batch size~$B$
  \REPEAT
  \FOR{minibatch $\{(x_j, y_j)\}_{j=1}^B$}
  	
    \FOR{$T$ steps}
        \STATE Draw $\delta_k \sim \fkr$,\ \ $k = 1,\dots,M$
    	\STATE $g_{\alpha_n} \gets 1 - \frac{1}{\rho M}
    		\sum\limits_{k=1}^M \indicator\big[ \ell(f_\theta(x_j+\delta_k), y_j) \geq \alpha_j \big]$
	    \STATE $\alpha_j \gets \alpha_j - \eta_\alpha g_{\alpha_j}$,\ \ for $n=1,\dots,B$
    \ENDFOR
    \STATE $g \gets \frac{1}{\rho M B} \sum\limits_{j,k}
    	\nabla_\theta \Big[ \ell\big( f_\theta(x_j+\delta_k), y_j \big) - \alpha_j \Big]_+$
    
    \STATE $\theta \gets \theta - \eta g$
  \ENDFOR
  \UNTIL{convergence}
\end{algorithmic}
\end{algorithm}

\subsection{A convex upper bound for the $\pmb{\rhosup}$}

Toward obtaining a practical algorithm for training probabilistically robust predictors, we first consider the relationship between probabilistic robustness and risk mitigation in portfolio optimization~\cite{krokhmal2002portfolio}. To this end, notice that the~$\rhosup$ is closely related to the inverse cumulative distribution function~(CDF): If~$F_{z}$ is the CDF of a random variable~$z$ with distribution~$\fkp$, then~$\rhosup_{z \sim \fkp} z = F_z^{-1}(\rho)$. For an appropriately-chosen distribution~$\fkp$, $F_z^{-1}(\rho)$ is known as the value-at-risk~(VaR) in the portfolio optimization literature. However, VaR is seldom used in practice due to its computational and theoretical limitations.  Indeed, VaR is often replaced with a tractable, convex upper bound known as the condition value-at-risk (CVaR)~\cite{rockafellar2000optimization, rockafellar2002conditional}.  Concretely, given a function~$f$ and a continuous distribution~$\fkp$, CVaR can be interpreted as the expected value of~$f$ on the tail of the distribution, i.e.,
\begin{equation}\label{eq:cvar}
    \CVaR_{\rho}(f;\fkp) = \E_{z \sim \fkp}\big[ f(z) \mid f(z) \geq F_z^{-1}(\rho) \big].
\end{equation}
It is straightforward to show that~$\CVaR_0(f;\fkp) = \E_{z\sim\fkp}[ f(z) ]$ and~$\CVaR_1(f;\fkp) = \esssup_{z\sim\fkp} f(z)$. In view of this property, it is not surprising that CVaR is an upper bound on~$\rhosup$, a result we summarize below:

\begin{proposition}[\cite{nemirovski2007convex}]\label{T:cvar-upper-bound}
CVaR is the tightest convex upper bound of~$\rhosup$, i.e.,
\begin{align}
    \rhosup_{z\sim\fkp} f(z) \leq \CVaR_{1-\rho} (f; \fkp)
\end{align}
with equality when~$\rho = 0$ or~$\rho = 1$.
\end{proposition}

% Hence, by training a classifier using the CVaR, not only are we optimizing its probabilistic robustness by minimizing the~$\rhosup$, but we also retain its average to worst case interpolation properties.

\subsection{Minimizing the conditional value at risk}

The main computational advantage of CVaR is that it admits the following convex, variational characterization:
\begin{equation}\label{eq:cvar-vartional}
    \CVaR_{\rho}(f;\fkp) = \inf_{\alpha\in\R}\ \alpha + \frac{\E_{z\sim\fkp}\big[ [f(z)-\alpha]_+ \big]}{1-\rho}
    	\text{.}
\end{equation}
Given this form, CVaR can be computed efficiently by using stochastic gradient-based techniques on~\eqref{eq:cvar-vartional}. This is the basis of the probabilistically robust training method detailed in Algorithm~\ref{alg:cvar-sgd}, which tackles the statistical problem
\begin{prob}[P-CVaR] \label{eq:p-cvar}
    \min_{h_p\in\calH}\ \E_{(x,y)} \Big[ \CVaR_{1-\rho} \left( \ell(h_p(x+\delta),y); \fkr \right) \Big]
\end{prob}
for parameterized, differentiable hypothesis classes~$\calH = \{f_\theta: \theta\in\Theta\}$.  Notice that like~\eqref{eq:p-rob},~\eqref{eq:p-cvar} is a \emph{composite} optimization problem involving an inner minimization over~$\alpha$ to compute CVaR and an outer minimization over~$\theta$ to train the predictor.  However, unlike~\eqref{eq:p-rob}, the inner problem in~\eqref{eq:p-cvar} is \emph{convex} regardless of $\calH$, and moreover the gradient of the objective in~\eqref{eq:cvar-vartional} can be computed in closed form.  To this end, in lines 5--6 of Algorithm~\ref{alg:cvar-sgd}, we compute CVaR via stochastic gradient descent (SGD) by sampling perturbations $\delta_k\sim\fkr$~\cite{thomas2019concentration}.  Then, in lines 9--10, we run SGD on the outer problem using an empirical approximation of the expectation based on a finite set of i.i.d.\ samples~$\{(x_j,y_j)\} \sim \fkD$ as in~\eqref{eq:erm}.

\begin{figure}
    \centering
    \includegraphics[width=0.9\columnwidth]{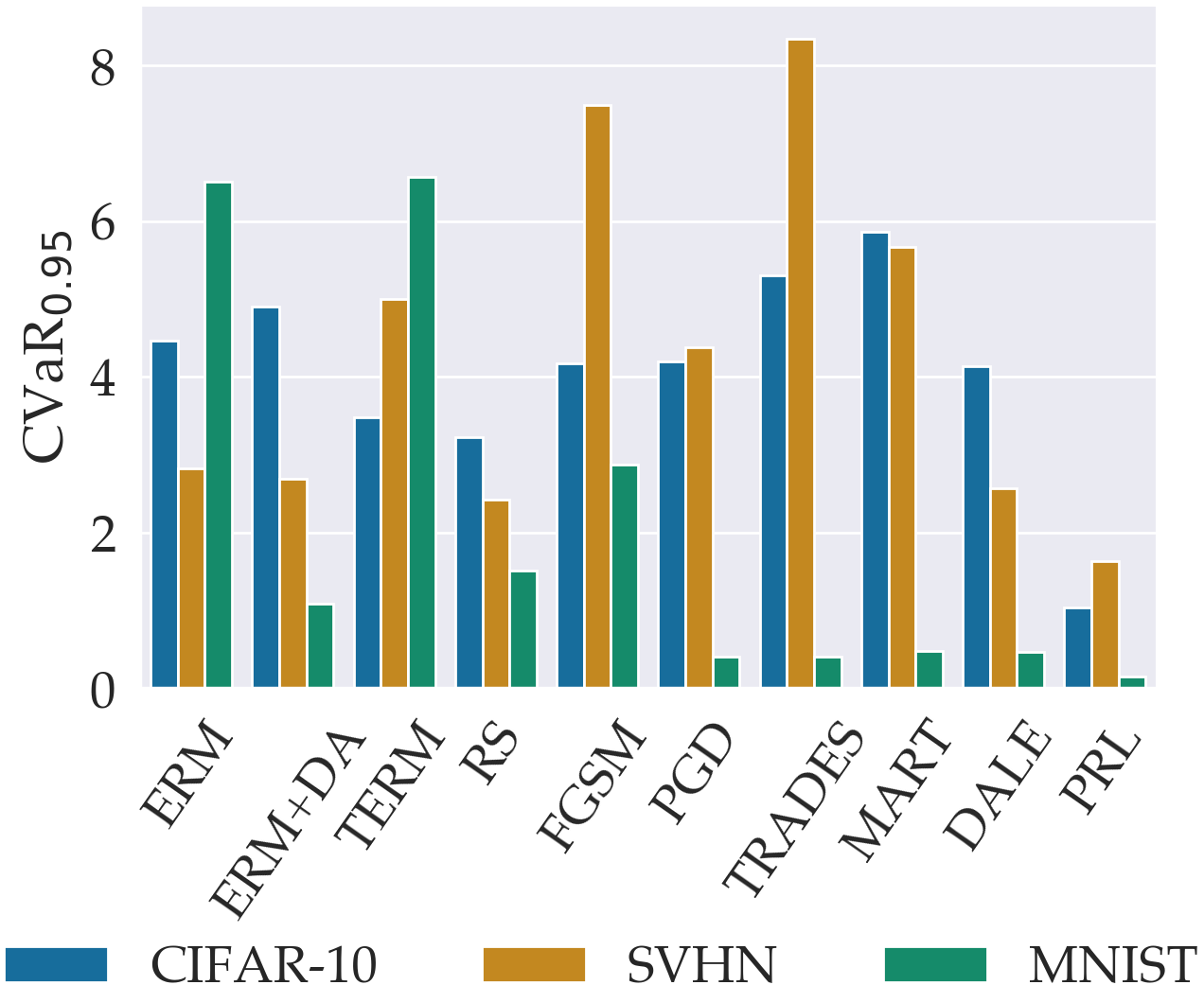}
    \caption{\textbf{CVaR as a metric for test-time robustness.}  We plot the conditional value at risk $\CVaR_{0.95}(\ell(h(x+\delta),y); \fkr)$ averaged over the test data points in CIFAR-10, SVHN, and MNIST respectively. Observe that PRL is more effectively able to minimize the objective in~\eqref{eq:p-cvar} than any of the baselines.}
    \label{fig:test-time-cvar}
    \vspace{-0.2em}
\end{figure}

\section{Experiments} \label{S:experiments}

We conclude our work by thoroughly evaluating the performance of the algorithm proposed in the previous section on three standard benchmarks: MNIST, CIFAR-10, and SVHN.  Throughout, we consider the perturbation set $\Delta = \{\delta\in\R^d : \norm{\delta}_\infty \leq \epsilon\}$ under the uniform distribution $\fkr$; for MNIST, we use $\epsilon=0.3$ and for CIFAR-10 and SVHN, we use $\epsilon=8/255$.  Further details concerning hyperparameter selection are deferred to the appendix.

\textbf{Baseline algorithms.}  We consider a range of baselines, including three variants of ERM: standard ERM~\cite{vapnik1999nature}, tilted ERM (denoted TERM)~\cite{li2020tilted,li2021tilted}, and ERM with data augmentation (denoted ERM+DA) wherein we run ERM on randomly perturbed instances.  Furthermore, we compare to the $L^q$ norm-based Hamiltonian Monte Carlo (N-HMC) method of~\cite{rice2021robustness}.  We also run various state-of-the-art adversarial training algorithms, including FGSM~\cite{goodfellow2014explaining}, PGD~\cite{madry2017towards}, TRADES~\cite{zhang2019theoretically}, MART~\cite{wang2019improving}, and DALE~\cite{robey2021adversarial}.

\textbf{Evaluation metrics.}  To evaluate the algorithms we consider, for each dataset we record the clean accuracy and the adversarial accuracy against a PGD adversary.  We also record the accuracy of each algorithm on perturbed samples in two ways.  Firstly, for each data point we randomly draw 100 samples from $\fkr$ and then record the average accuracy across perturbed samples $x+\delta$; we denote these accuracies by ``Aug.'' in the relevant tables.  And secondly, to explicitly measure probabilistic robustness, we propose the following \emph{quantile accuracy} metric, the form of which follows directly from the probabilistically robust 0-1 loss defined in~\eqref{eq:01-loss-unconst}:
\begin{align}
    \text{ProbAcc}(\rho) =  \indicator\left[ \Prob_{\delta\sim\fkr} \left[ h(x+\delta)= y \right] \geq 1-\rho \right]. \label{eq:quant-acc}
\end{align}
In words, this metric describes the proportion of instances which are probabilistically robust with tolerance level $\rho$, and therefore this will be our primary metric for evaluating probabilistic robustness for a given tolerance level $\rho$.

\begin{table}[t!]
    \centering
    \caption{\textbf{Classification results for CIFAR-10.}}
    \vspace{0.05in}
    \resizebox{\columnwidth}{!}{%
    \begin{tabular}{ccccccc} \toprule
     \multirow{2}{*}{Algorithm} & \multicolumn{3}{c}{Test Accuracy} & \multicolumn{3}{c}{ProbAcc($\rho)$} \\ \cmidrule(lr){2-4} \cmidrule{5-7}
         & Clean & Aug.\ & Adv.\ & 0.1 & 0.05 & 0.01 \\ \midrule
         ERM & \textbf{94.38} & 91.31 & 1.25 & 86.35 & 84.20 & 79.17 \\
         ERM+DA & 94.21 & 91.15 & 1.08 & 86.35 & 84.15 & 79.19 \\
         TERM & 93.19 & 89.95 & 8.93 & 84.42 & 82.11 & 76.46 \\
         N-HMC & 85.07 & 84.41 & 3.24 & 79.50 & 77.96 & 74.76 \\
         FGSM & 84.96 & 84.65 & 43.50 & 83.76 & 83.50 & 82.85 \\
         PGD & 84.38 & 84.15 & 47.07 & 83.18 & 82.90 & 82.32 \\
         TRADES & 80.42 & 80.25 & 48.54 & 79.38 & 79.12 & 78.65 \\
         MART & 81.54 & 81.32 & 48.90 & 80.44 & 80.21 & 79.62 \\
         DALE & 84.83 & 84.69 & \textbf{50.02} & 83.77 & 83.53 & 82.90 \\ \midrule
         \rowcolor{Gray} PRL & 93.82 & \textbf{93.77} & 0.71 & \textbf{91.45} & \textbf{90.63} & \textbf{88.55} \\ \bottomrule
    \end{tabular}}
    \label{tab:cifar-accs}
% \end{table}

% \begin{table}[]
    \centering
    \caption{\textbf{Classification results for SVHN.}}
    \vspace{0.05in}
    \resizebox{\columnwidth}{!}{%
    \begin{tabular}{ccccccc} \toprule
     \multirow{2}{*}{Algorithm} & \multicolumn{3}{c}{Test Accuracy} & \multicolumn{3}{c}{ProbAcc($\rho)$} \\ \cmidrule(lr){2-4} \cmidrule{5-7}
         & Clean & Aug.\ & Adv.\ & 0.1 & 0.05 & 0.01 \\ \midrule
         ERM & 94.44 & 94.28 & 2.72 & 92.16 & 91.40 & 89.42 \\
         ERM+DA & 94.69 & 94.43 & 2.08 & 92.65 & 92.01 & 89.92 \\
         TERM & 91.85 & 91.58 & 18.33 & 89.01 & 88.04 & 85.85 \\
         N-HMC & 90.32 & 90.55 & 18.30 & 88.79 & 87.61 & 86.12 \\
         FGSM & 80.69 & 85.55 & 32.82 & 80.18 & 78.02 & 74.87 \\
         PGD & 91.19 & 91.29 & 44.89 & 90.15 & 89.68 & 83.82 \\
         TRADES & 86.16 & 86.47 & \textbf{54.89} & 85.09 & 84.76 & 83.82 \\
         MART & 90.20 & 90.44 & 45.23 & 89.81 & 88.82 & 84.32 \\
         DALE & 93.85 & 93.72 & 51.98 & 92.52 & 91.08 & 89.19 \\ \midrule
         \rowcolor{Gray} PRL & \textbf{95.00} & \textbf{94.81} & 3.11 & \textbf{93.28} & \textbf{92.97} & \textbf{91.74} \\ \bottomrule
    \end{tabular}}
    \label{tab:svhn-accs}
% \end{table}

% \begin{table}[]
    \centering
    \caption{\textbf{Classification results for MNIST.}}
    \vspace{0.05in}
    \resizebox{\columnwidth}{!}{%
    \begin{tabular}{ccccccc} \toprule
     \multirow{2}{*}{Algorithm} & \multicolumn{3}{c}{Test Accuracy} & \multicolumn{3}{c}{ProbAcc($\rho)$} \\ \cmidrule(lr){2-4} \cmidrule{5-7}
         & Clean & Aug.\ & Adv.\ & 0.1 & 0.05 & 0.01 \\ \midrule
        ERM & 99.37 & 98.82 & 0.01 & 97.96 & 97.96 & 96.66 \\
        ERM+DA & \textbf{99.42} & 99.13 & 5.23 & 98.46 & 98.12 & 97.30 \\
        TERM & 99.20 & 98.55 & 11.27 & 97.15 & 96.42 & 94.15 \\
        N-HMC & 99.33 & \textbf{99.25} & 3.91 & 98.85 & 98.71 & 98.23 \\
        FGSM & 98.86 & 98.72 & 19.34 & 98.00 & 97.83 & 97.25 \\
        PGD & 99.16 & 99.10 & 94.45 & 99.05 & 98.63 & 98.34 \\
        TRADES & 99.10 & 99.04 & \textbf{94.76} & 98.71 & 98.61 & 98.33 \\
        MART & 98.94 & 98.98 & 94.13 & 98.59 & 98.39 & 97.98 \\ \midrule
        \rowcolor{Gray} PRL & 99.32 & \textbf{99.25} & 26.03 & \textbf{99.27} & \textbf{99.01} & \textbf{98.54} \\ \bottomrule
    \end{tabular}}
    \label{tab:mnist-accs}
    \vspace{-1.0em}
\end{table}

\textbf{Clean, robust, and quantile accuracies.} In Tables~\ref{tab:cifar-accs}--\ref{tab:mnist-accs}, we record the clean, robust, and probabilistic error metrics described above for PRL and a range of baselines.  Throughout, the value of $\rho$ was chosen by cross-validation; see Appendix~\ref{app:training} for details. Given these results, several remarks are in order.  Firstly, across each of these tables, it is clear that the PRL algorithm does not incur the same degradation in nominal performance as does adversarial training; indeed, on CIFAR-10 and MNIST, the clean accuracy of PRL is within one percentage point of ERM, and for SVHN, the clean accuracy of PRL surpasses that of ERM.  A second observation is that across these datasets, PRL offers significant improvements in the ProbAcc$(\rho)$ metric.  This improvement manifests most clearly on CIFAR-10, wherein PRL improves by more than six percentage points over all baseline algorithms for $\rho=0.01$.  Moreover, the gap between the ProbAcc of PRL and that of the baselines increases as $\rho$ decreases, indicating that PRL is particularly effective for more stringent robustness requirements.

We also highlight the fact that PRL consistently outperforms both TERM and~\cite{rice2021robustness} on the clean, augmented, and quantile accuracy metrics.\footnote{We selected $t$ and $q$ for TERM and N-HMC by cross-validation; see Appendix~\ref{app:training} for details.}  This demonstrates that PRL facilitates a strong empirical trade-off between robustness and accuracy relative to other methods that seek to interpolate between the average and worst case.

\textbf{CVaR as a metric for test-time robustness.}  As we showed in Section~\ref{S:algorithm}, $\CVaR_{1-\rho}$ is an upper bound for the $\rhosup$.  In this way, CVaR can be used as a surrogate for assessing the test-time robustness of trained classifiers.  To this end, in Figure~\ref{fig:test-time-cvar} we plot $\CVaR_{0.95}(\ell(h(x+\delta),y); \fkr)$ averaged over the test data on CIFAR-10, SVHN, and MNIST.  This plot shows that PRL displays significantly lower values of $\CVaR_{0.95}$ among all of the algorithms we considered, which reinforces the message from Tables~\ref{tab:cifar-accs}-\ref{tab:mnist-accs} that PRL is most successful at imposing probabilistic robustness.

\begin{figure}
    \centering
    \includegraphics[width=0.9\columnwidth]{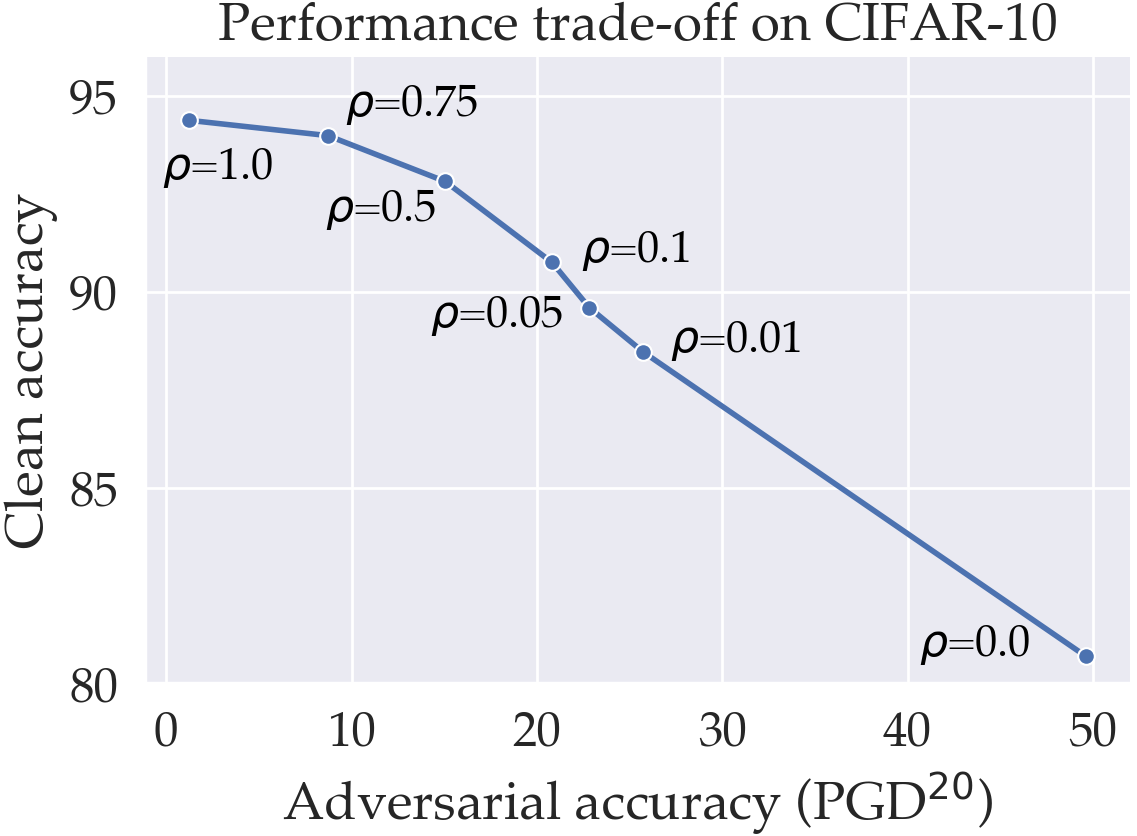}
    \caption{\textbf{Trade-offs between adversarial and clean accuracy.}  By sweeping over $\rho$, we show that our approach bridges the average and worst case by trading-off clean and adversarial accuracy.  Thus, as $\rho$ decreases, trained classifiers improve robustness to adversarial perturbations at the cost of decreasing clean performance.}
    \label{fig:acc-vs-rob}
    \vspace{-0.5em}
\end{figure}

\textbf{Ablation study: the role of $\rho$ in Algorithm~\ref{alg:cvar-sgd}.}  In Section~\ref{S:algorithm}, we claimed that our algorithm interpolates between the average- and worst-case problems in~\eqref{eq:p-avg} and~\eqref{eq:p-rob} respectively.  To verify this claim, we study the trade-off between nominal accuracy and adversarial accuracy for varying values of $\rho$.  In Figure~\ref{fig:acc-vs-rob}, we show that as $\rho$ decreases, our algorithm improves adversarial accuracy at the cost of degrading nominal performance.

\section{Conclusion}

In this paper, motivated by the brittleness of ERM and the conservatism of adversarial training, we proposed a new framework called probabilistically robust learning in which robustness is enforced with high probability over perturbations rather than in the worst case.  Our analysis of this framework showed that PRL provably mitigates the well-known trade-off between robustness and accuracy and can have a sample complexity that is exponentially lower than that of adversarial training.  We also proposed an algorithm motivated by risk-aware optimization which shows strong performance on a variety of metrics designed to evaluate intermediate robustness between the average and worst case.  

There are numerous directions for future work.  One fruitful direction is to further explore the use of risk-aware optimization in other areas of learning.  Although these techniques have been applied to problems in reinforcement learning~\cite{chow2015risk} and active learning~\cite{curi2019adaptive}, there is ample opportunity to use these methods to improve performance in fields like domain generalization, domain adaptation, and fair learning.  Another exciting direction is to devise tighter and/or more efficient schemes to optimize the $\rhosup$ in~\eqref{eq:p-prl}.  Furthermore, although we focused on additive, norm-bounded perturbations in this paper, there is ample opportunity to extend this framework to deal with more general distribution shifts, such as those studied in~\cite{robey2020model,robey2021model,wong2020learning,gowal2020achieving,zhou2022deep}.

A final direction of interest is to study the connections between alternative definitions of robustness, including both~\cite{li2020tilted,li2021tilted} and~\cite{rice2021robustness} as well as existing notions of astuteness~\cite{wang2018analyzing} and of neighborhood-preserving Bayes optimal classifiers~\cite{bhattacharjee2021consistent} for non-parametric methods.  Indeed, a unification of these robustness frameworks may represent a significant advance in our understanding of the robustness of machine learning models.

\section*{Acknowledgements}

The research of Hamed Hassani and Alexander Robey is supported by NSF grants 1837253, 1943064, AFOSR grant FA9550-20-1-0111, DCIST-CRA, and the AI Institute for Learning-Enabled Optimization at Scale (TILOS).  George J.\ Pappas and Alexander Robey are also supported by NSF grant 2038873 as well as the NSF-Simons Foundation’s Mathematical and Scientific Foundations of Deep Learning (MoDL) program on Transferable, Hierarchical, Expressive, Optimal, Robust, Interpretable NETworks (THEORINET).

\newpage

\newpage

\bibliography{bibliography}
\bibliographystyle{icml2022}

\newpage
\onecolumn
\appendix

\section{Proofs concerning trade-offs in binary classification and linear regression}

\subsection{Binary classification under a Gaussian mixture model}

\begin{figure}
    \centering
    \includegraphics[width=0.3\textwidth]{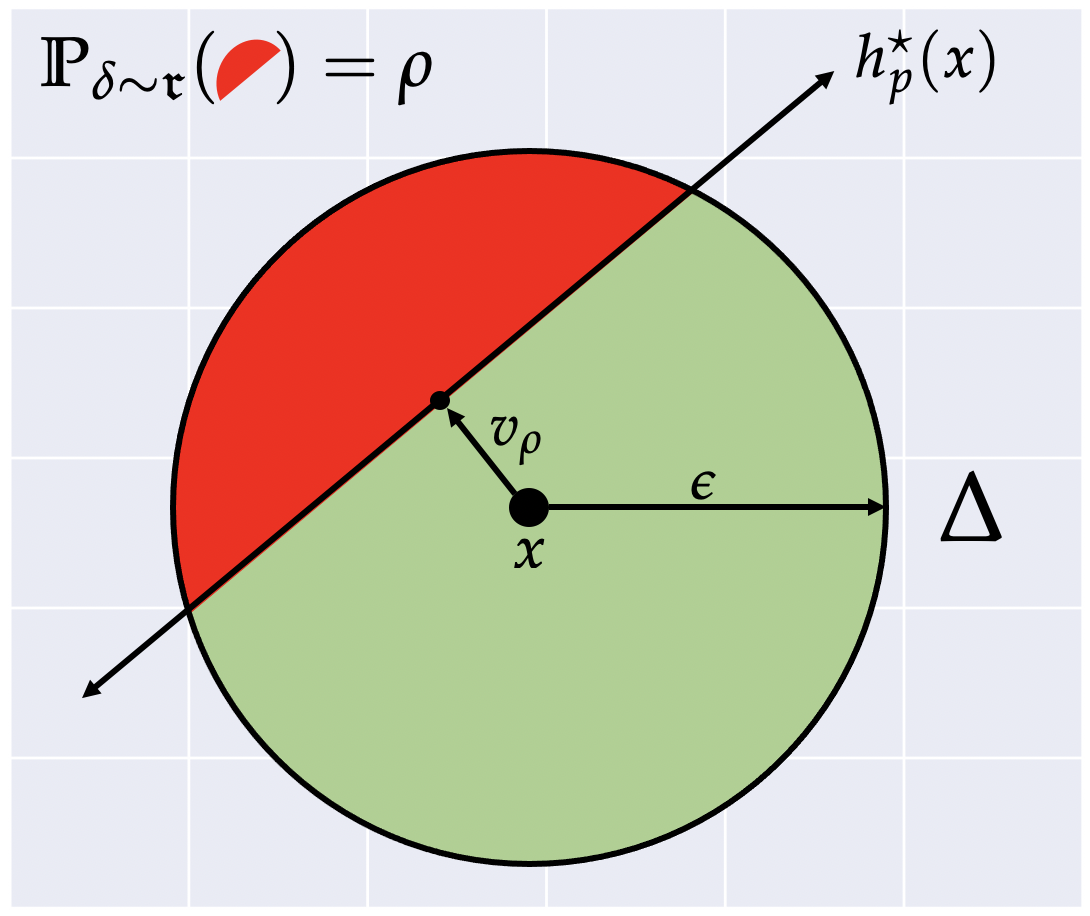}
    \caption{\textbf{Spherical cap of an $\ell_2$-ball with radius $\epsilon>0$ in two-dimensions.}}
    \label{fig:spherical-cap}
\end{figure}

In this subsection, we provide proofs for the results for binary classification in Section~\ref{S:tradeoff}.  In general, our proof of the closed form expression for $h^\star_p$ in Proposition~\ref{prop:opt-beta-rob-gaussian} follows along the same lines as the proof of Theorem 4.2 in~\cite{dobriban2020provable}.  In particular, our contribution is to generalize the proof techniques to the setting of probabilistic robustness, thereby subsuming the results in~\cite{dobriban2020provable} as a special case when $\rho=0$.

\begin{lemma}
For any $\rho\in[0, 1/2]$, it holds that among all linear classifiers,
\begin{align}
    h^\star(x) =
        \sign\left(x^\top \mu\left(1 - \frac{v_\rho}{\norm{\mu}_2} \right)_+ - \frac{q}{2}\right)
\end{align}
is optimal for~\eqref{eq:p-prl}, where $v_\rho$ is the distance from the center of $\Delta$ to a spherical cap of volume $\rho$.
\end{lemma}

\begin{proof}
To begin, observe that the the probabilistically robust risk $\PR(h_p;\rho)$ can be written in the following way:
\begin{align}
    \PR(h_p;\rho) &= \E_{(x,y)\sim\fkD} \Big[\indicator\big[
		\Prob_{\delta \sim \fkr} \left[ h(x+\delta) \neq y \right] > \rho
	\big]\Big] \\
	&= \Prob_{(x,y)\sim\fkD}\Big[ \Prob_{\delta \sim \fkr} \left[ h(x+\delta) \neq y \right] > \rho
	 \Big] \\
	 &= \Prob[y=+1] \cdot \Prob_{x|y=+1}\Big[\Prob_{\delta \sim \fkr} \left[ h(x+\delta) =-1 \right] > \rho\Big] \\
	 &\qquad + \Prob[y=-1] \cdot \Prob_{x|y=-1}\Big[\Prob_{\delta \sim \fkr} \left[ h(x+\delta) =+1 \right] > \rho\Big].
\end{align}
Note that it is enough to solve this problem in one dimension, as the problem in $d$-dimensions can be easily reduced to a one-dimensional problem.  Thus, our goal is to find the value of a threshold $c$ that minimizes the probabilistically robust risk.  In this one-dimensional case, the probabilistically robust risk can be written as
\begin{align}
    \PR(h_p;\rho) &= \pi \cdot \Prob_{x|y=+1}[x\leq c+\rho] + (1-\pi) \cdot \Prob_{x|y=-1}[x\geq c-\rho] \\
    &= \pi \cdot \Prob_{x|y=+1}[x-\rho\leq c] + (1-\pi) \cdot \Prob_{x|y=-1}[x+\rho\geq c].
\end{align}
Recall that as $x|y\sim\mathcal{N}(y\mu, \sigma^2 I)$.  Therefore,
\begin{align}
    \PR(h_p;\rho) = \pi \cdot \Prob_{x\sim\calN(\mu-\rho, \sigma^2 I)}[ x \leq c] + (1-\pi) \cdot \Prob_{x\sim\calN(-\mu+\rho, \sigma^2 I)}[ x \geq c]
\end{align}
This is exactly the same as the problem of non-robust classification between tow Gaussians with means $\mu^\prime=\mu-\rho$ and $-\mu^\prime$ (by assumption, we have that $\mu\geq 0$).  As is well known (see, e.g.,~\cite{anderson1958introduction}), the optimal classifier for this setting is
\begin{align}
    h^\star_p(x) = \sign[ x\cdot(\mu-\rho) - q/2] \label{eq:opt-lin}
\end{align}
where $q = \ln[(1-\pi)/\pi]$.  And indeed, when moving from the one-dimensional case, one need only recognize that a linear classifier which ignores a set of volume $\rho$ in $\Delta$ will
create a spherical cap of volume $\rho$ in $\Delta$.  This is illustrated by the red region in Figure~\ref{fig:spherical-cap}.  The form of $h^\star_p(x)$ given in~\eqref{eq:mix-gaussians-opt-classifier} follows from~\eqref{eq:opt-lin} as a direct analog for the $d$-dimensional case.
\end{proof}

To prove the second part of Proposition~\ref{prop:opt-beta-rob-gaussian}, we seek to characterize the distance $v_\rho$ from the center of $\Delta$ to a spherical cap of volume $\rho$.  To this end, we have the following result.

\begin{lemma}
Let $B(0,\gamma) = \{\delta\in\Delta : \norm{\delta}_2\leq \gamma\}$ for any number $\gamma >0$.  Define $v_\rho$ to be the distance from the origin such that the fraction of the volume of the corresponding spherical cap is $\rho$ (see Figure~\ref{fig:spherical-cap}).  Then we have that
\begin{align}
    v_\rho = \begin{cases}
        \epsilon &\quad \rho = 0 \\
        \frac{\epsilon}{\sqrt{d}} \Phi^{-1}(1-\rho)(1-o_d(1)) &\quad \rho\in(0,1/2]
    \end{cases}
\end{align}
\end{lemma}

\begin{proof}
By inspection, the result is clear for $\rho=0$.  Thus, we consider the case when $\rho\in[0,1/2]$.  Note that for any number $\gamma < \epsilon$, we have that
\begin{align}
    \frac{\Prob_{\delta\sim\fkr}(B(0, \epsilon)) - \Prob_{\delta\sim\fkr}(B(0, \epsilon-\gamma))}{\Prob_{\delta\sim\fkr}(B(0, \epsilon))} = 1 - \left(1-\frac{\gamma}{\epsilon}\right)^d.
\end{align}
As a result, by taking $\gamma = \epsilon \cdot (\ln(d)/d)$ we obtain
\begin{align}
    1 - \left(1-\frac{\gamma}{\epsilon}\right)^d &= 1 - \exp\left( d\ln\left(1-\frac{\gamma}{\epsilon}\right)\right) \\
    &= 1 - \exp\left( -\frac{d\gamma}{\epsilon} + O\left(d\left(\frac{\gamma}{\epsilon}\right)^2\right)\right) \\
    &= 1 - O(1/d).
\end{align}
In this way, we have shown that the uniform distribution over any ball centered at the origin can be approximated up to $o_d(1)$ by the uniform distribution on the sphere.  

Now let $(X_1, \cdots, X_n)$ be the random vector generated by uniformly sampling a point on the sphere of radius $\epsilon$.  We note that, up to $o_d(1)$ terms, the distribution of each of the coordinates, e.g. $X_1$, is $\epsilon Z/\sqrt{d}$, where $Z$ is the normal random variable.  Again, up to $o_d(1)$ terms, i.e. when the dimension grows large, the volume of the spherical cap at distance $v_\rho$ can be approximated by 
\begin{align}
    \Prob(X_1 \geq v_\rho) = \Prob\left(Z \frac{\epsilon}{\sqrt{d}} \geq v_\rho\right) = 1 - \Phi\left(v_\rho\, \frac{\sqrt{d}}{\epsilon}\right).  
\end{align}
where $\Phi$ denotes the Gaussian CDF.  As a result, for the RHS of the above relation to be equal to $\rho$, we must have 
\begin{align}
    v_\rho = \frac{\epsilon \,\Phi^{-1}(1 - \rho)}{\sqrt{d}}. 
\end{align}
This concludes the proof.
\end{proof}

From the above lemma, we can conclude the following phase-transition behavior. When $\rho = 0$, there is a constant gap between the adversarially robust risk $\AR(h_r^\star)$ and the best attainable clean risk $SR(h^\star_\textup{Bayes})$. Indeed, this gap does not vanish as the dimension $d$ grows, resulting in a non-trivial trade-off between adversarial robustness and accuracy. However, for $\rho > 0$, the gap between the probabilistically robust accuracy $\PR(h_p^\star;\rho)$ and the clean risk is of the form $\PR(h_p^\star; \rho) - \SR(h_\text{Bayes}^\star) = O(1/\sqrt{d})$, and as a result, as the dimension $d$ grows, the trade-off between robustness and accuracy vanishes  (a blessing of high dimensions).

\subsection{Linear regression with Gaussian features} \label{app:lin-regression}

We next consider the setting of linear regression, wherein it is assumed that there exists an underlying parameter vector $\theta_0\in\Theta\subset\R^d$, and that the data is subsequently generated according to the following model:
\begin{align}
    x\sim\calN(0, I_d), \quad y = \theta_0^\top x + z, \quad z\sim\calN(0, \sigma^2) \label{eq:lin-reg-data}
\end{align}
where $\sigma>0$ is a fixed noise level. Furthermore, we consider hypotheses of the form $f_\theta(x) = \theta^\top x$ for $\theta\in\Theta$, and we use the squared loss $\ell(f_\theta(x),y) = (f_\theta(x) - y)^2 = (\theta^\top x - y)^2$.  In this setting, it is straightforward to calculate that at optimality $\SR(f_{\theta}) = \sigma^2$, which is achieved for~$\theta=\theta_0$.  Moreover, in the more general probabilistic robustness setting, we have the following complementary result:
\begin{proposition}
Suppose that the data is distributed according to~\eqref{eq:lin-reg-data}.  Let $\theta^\star\in\Theta$ denote the optimal solution obtained by solving~\eqref{eq:p-prl} over $\Theta$.  Then for any $\rho > 0$,
\begin{align}
    \PR(f_{\theta^\star}; \rho) - \SR(f_{\theta_0}) = \begin{cases}
        O(1/\sqrt{d}) &\: \rho > 0 \\
        O(1) &\: \rho = 0
    \end{cases}
\end{align}
\end{proposition}
In this way, as in the previous subsection, it holds that for any $\rho > 0$, the gap between probabilistic robustness and nominal performance vanishes in high dimensions.  On the other hand, as was recently shown in~\cite{javanmard2020precise}, there exists a non-trivial gap between adversarial robustness and clean accuracy that does not vanish to zero by increasing the dimension in this setting.

To prove this result, we consider the following variational form of the problem in~\eqref{eq:p-prl}:
\begin{prob}\label{eq:p-regression}
    &\min_{\theta \in \mathbb{R}^d, \: t\in L^1} &&\E_{(x,y)\sim\fkD} [t(x,y)] \\
    &\text{subject to} &&\Prob_{\delta\sim\fkr} \left\{ (\theta^{\sf T} (x+\delta) - y)^2 \leq t(x,y)\right\} \geq 1-\rho \quad\forall (x,y)\in\Omega.
\end{prob}
where $L^1$ denotes the space of Lebesgue integral functions.  We can then characterize $t(x,y)$ as follows:
\begin{lemma}
We have the following characterization for $t(x,y)$:
\begin{equation*}
t(x,y) = \left\{
\begin{array}{rl}
(|\theta^{\sf T} x - y| + \epsilon ||\theta||_2)^2 & \text{if } \rho  = 0,\\
(\theta^{\sf T} x - y)^2 + \frac{\epsilon^2 ||\theta||^2 (\theta^{\sf T} x - y)}{\sqrt{d}}  \left (Q^{-1}(1-\rho)  + o_d(1) \right ) & \text{if } \rho \in (0,  1],
\end{array} \right.
\end{equation*}
Almost surely for any $(x,y)$.

\end{lemma}

\begin{proof}
Let's first consider the case in which $\rho > 0$. Since $\delta \sim \mathbb{P}_{\delta\sim\fkr}$ is the uniform distribution over the Euclidean ball of radius $\epsilon$, we know that for any $\theta \in \mathbb{R}^d$ we have
$$ \theta^{\sf T} \delta \stackrel{d}{\to} \mathcal{N}(0, \frac{\epsilon^2 ||\theta||_2^2}{d}),$$
where the convergence is in distribution. This is because the uniform distribution over the Euclidean ball of radius $\epsilon$ converges to the Gaussian distribution $N(0, \frac{\epsilon^2}{d} I_d)$. As a result, up to $o_d(1)$ terms, we have $\theta^{\sf T} \delta \sim \frac{\epsilon ||\theta||_2}{\sqrt{d}} Z $, where $Z$ is the normal random variable.  
\begin{align*}
   &\Prob_{\delta\sim\fkr}\left\{ (\theta^{\sf T} (x+\delta) - y)^2 \leq t(x,y)\right\}  \\ 
   &=  
    \Prob_Z\left\{ \frac{\epsilon^2 ||\theta||_2^2}{d} Z^2 + 2(\theta^{\sf T} x - y)\frac{\epsilon ||\theta||_2}{\sqrt{d}} Z + (\theta^{\sf T} x - y)^2 \leq t(x,y)  \right\} + o_d(1)\\
    & = \Prob_Z\left\{ 2(\theta^{\sf T} x - y)\frac{\epsilon ||\theta||_2}{\sqrt{d}} Z + (\theta^{\sf T} x - y)^2 \leq t(x,y)  \right\} + o_d(1) \\
    & = \Prob_Z\left\{  \frac{\epsilon ||\theta||_2}{\sqrt{d}} Z + (\theta^{\sf T} x - y)^2 \leq t(x,y)  \right\} + o_d(1) \\
    & = \Prob_Z\left\{   Z  \leq \sqrt{d}\frac{t(x,y) -  (\theta^{\sf T} x - y)^2}{2\epsilon ||\theta||_2 (\theta^{\sf T} x - y)} \right\} + o_d(1) \\
    & =  Q\left(  \sqrt{d}\frac{t(x,y) -  (\theta^{\sf T} x - y)^2}{2\epsilon ||\theta||_2 (\theta^{\sf T} x - y)} \right) + o_d(1) 
\end{align*}
where $Q$ is the quantile function for $Z$, i.e. the inverse of the normal CDF.  As a result, equating the above with $1-\rho$ we obtain
$$ t(x,y) = (\theta^{\sf T} x - y)^2 + \frac{2\epsilon ||\theta||_2 (\theta^{\sf T} x - y)}{\sqrt{d}} \left( Q^{-1}(1 - \rho) + o_d(1) \right).  $$

For the case $\beta = 0$, this is a simple optimization problem where the closed form solution provided in the lemma is its solution.
\end{proof}

From the above lemma, it is easy to see that for $\theta = \theta_0$, the value of the objective in \eqref{eq:regression} becomes $\sigma^2 + O(1/\sqrt{d})$ for any value of $\rho > 0$. This means that the gap between the probabilistically roust risk and the clean risk  is of the form $\PR(h_p^\star;\rho) - \SR(h_r) = O(1/\sqrt{d})$, and as a result, as the dimension $d$ grows, the trade-off between probabilistic robustness and accuracy vanishes. We further remark that for $\rho = 0$, i.e. the adversarial setting, there exists a non-trivial gap between the robust and clean accuracies that does not vanish to zero by increasing dimension. This is indeed clear from the above lemma, and it has also been shown in \cite{javanmard2020precise}. 
\newpage
\section{Learning theory proofs}\label{A:learning-theory-proofs}

We restate Proposition~\ref{T:sample_complexity} formally to detail what we mean by \emph{sample complexity}. Note that this is exactly the number of sample required for PAC learning.

\begin{proposition}
Consider the probabilistically robust learning problem~\eqref{eq:p-prl} with the 0-1 loss function and a robust measure~$\fkr$ fully supported on~$\Delta$ and absolutely continuous with respect to the Lebesgue measure~(i.e., non-atomic). Let~$P_r^\star$ be its optimal value and consider its empirical version
\begin{prob*}
	\hat{P}_r^\star = \min_{h_p \in \calH} \: \frac{1}{N} \sum_{n=1}^N
		\left[ \rhosup_{\delta \sim \fkr}\ \ell\big( h_p(x_n+\delta), y_n \big) \right]
		\text{,}
\end{prob*}
based on i.i.d.\ samples~$(x_n,y_n) \sim \fkD$. For any threshold~$\rho_o \in (0,0.5)$, there exists a hypothesis class~$\calH_o$ such that the sample complexity of probabilistically robust learning, i.e., the number of samples~$N$ needed for~$\big\vert P_r^\star - \hat{P}_r^\star \big\vert \leq \epsilon$ with high probability, is
\begin{equation*}
	N = \begin{cases}
		\Theta\big( \log_2(1/\rho_o)/\epsilon^2 \big) \text{,} &\rho = 0
		\\
		\Theta(1/\epsilon^2) \text{,} &\rho_o \leq \rho \leq 1-\rho_o
	\end{cases}
\end{equation*}
In particular, $\Theta(1/\epsilon^2)$ is the sample complexity of~\eqref{eq:nom-training}.
\end{proposition}

\begin{proof}
Let us begin by reducing the task of determining the sample complexity of these problems to that of determining the VC dimension of their objectives:

\begin{lemma}[{\citep[Thm.\ 6.8]{shalev2014understanding}}]\label{L:sample_complexity}
	Consider
	\begin{equation}
	\begin{aligned}
	    P^\star = \min_{h \in \calH}& &&\E_{(x,y)\sim\fkD}\!\big[ g(h, x, y) \big]
        \\
	    \hat{P^\star} = \min_{h \in \calH}& &&\frac{1}{N} \sum_{n = 1}^N g(h, x_n, y_n)
	\end{aligned}
	\end{equation}
	where~$g: \calH \times \calX \times \calY \to \{0,1\}$ and the samples~$\{(x_n,y_n)\} \sim \fkD$ are i.i.d. The number of samples~$N$ needed for~$\big\vert P^\star - \hat{P}^\star \big\vert \leq \epsilon$ with probability~$1-\delta$ over the sample set~$\{(x_n,y_n)\}$ is
	\begin{equation*}
		C_1 \frac{d_\text{VC} + \log(1/\delta)}{\epsilon^2} \leq N \leq C_2 \frac{d_\text{VC} + \log(1/\delta)}{\epsilon^2}
			\text{,}
	\end{equation*}
	for universal constants~$C_1,C_2$. The VC dimension~$d_\text{VC}$ is defined the largest~$d$ such that~$\Pi(d) = 2^d$ for the \emph{growth function}
	\begin{equation*}
		\Pi(d) = \max_{\{(x_n,y_n)\} \subset (\calX \times \calY)^d}\ \abs{ \calS\big( \{(x_n,y_n)\} \big) }
			\text{,}
	\end{equation*}
	where~$\calS\big( \{(x_n,y_n)\} \big) = \big\{u \in \{0,1\}^m \mid \exists h \in \calH \ \text{such that}\ \allowbreak u_n = \ell\big( h(x_n), y_n \big) \big\}$.
\end{lemma}

We now proceed by defining~$\calH_o$ using a modified version of the construction in~\citep[Lemma~2]{montasser2019vc}. Let~$m = \ceil{\log_2(1/\rho_o)} + 1$ and pick~$\{c_1,\dots,c_m\} \in \calX^m$ such that, for~$\Delta_i = c_i + \Delta$, it holds that~$\Delta_i \cap \Delta_j = \emptyset$ for~$i \neq j$. Within each~$\Delta_i$, define~$2^{m-1}$ disjoint sets~$\calA_{i}^b$ of measure~$\fkr(\calA_{i}^b) \leq \rho_o/m$ labeled by the binary digits~$b \in \{0,1\}^m$ whose~$i$-th digit is one. In other words, $\Delta_1$ contains sets with signature~$1 \mathsf{b}_2 \dots \mathsf{b}_m$ and~$\Delta_3$ contains sets with signature~$\mathsf{b}_1 \mathsf{b}_2 1 \dots \mathsf{b}_m$. Observe that there are indeed~$2^{m-1}$ sets~$\calA_{i}^b$ within each~$\Delta_i$, that their signatures span all possible~$m$-digits binary numbers, and there are at most~$m$ sets with the same signature~(explicitly, for~$b = 11\dots1$). Additionally, note that it is indeed possible to fit the~$\calA_{i}^b$ inside each~$\Delta_i$ given that
\begin{equation*}
	2^{m-1} \cdot \dfrac{\rho_o}{m} < \dfrac{2^{\log_2(1/\rho_o) + 1} \rho_o}{\log_2(1/\rho_o) + 1}
		\leq \dfrac{2}{\log_2(1/\rho_o) + 1} \leq 1
\end{equation*}
for~$\rho_o < 0.5$, where we used the fact that~$\log_2(1/\rho_o) \leq \ceil{\log_2(1/\rho_o)} < \log_2(1/\rho_o) + 1$.

We can now construct the hypothesis class~$\calH_o = \{h_b \mid b \in \{0,1\}^m\}$ by taking
\begin{equation}\label{eq:hypothesis}
	h_b(x) = \begin{cases}
		1 \text{,} &x \notin \bigcup_{i = 1}^m \calA_{i}^b
		\\
		0 \text{,} &x \in \bigcup_{i = 1}^m \calA_{i}^b
	\end{cases}
\end{equation}

Let us proceed first for the probabilistically robust loss~$g\big( h, x, y \big) = \rhosup_{\delta \sim \fkr}\ \indicator\big[ h(x+\delta) \neq y \big]$.

We begin by showing that if~$\rho = 0$, then~$d_\text{VC} > m$. Indeed, consider the set of points~$\{(c_i,1)\}_{i=1,\dots,m} \subset (\calX \times \calY)^m$. In this case, the cardinality of~$\calS(\{(c_i,1)\})$ is~$2^m$, i.e., this set can be shattered by~$\calH$. Indeed, for any signature~$b \in \{0,1\}^m$, we have
\begin{equation*}
	\rhosup_{\delta \sim \fkr}\ \indicator[h_b(c_i + \delta) \neq 1] =
	\esssup_{\delta \sim \fkr}\ \indicator[h_b(c_i + \delta) \neq 1] = b_i
		\text{,} \quad \text{for } i = 1,\dots,m
		\text{,}
\end{equation*}
since~$h_b(c_i + \delta) = 0$ for all~$c_i + \delta \in \calA_i^b$, a set of positive measure. Using Lemma~\ref{L:sample_complexity}, we therefore conclude that~$N = \Theta\big( m/\epsilon^2 \big)$.

Let us now show that~$d_\text{VC} = 1$ for~$\rho \geq \rho_o$ by showing that~$\Pi(2) < 4$. To do so, we take two arbitrary points~$(x_1, y_1), (x_2, y_2) \in \calX \times \calY$ and proceed case-by-case. To simplify the exposition, let~$\calA = \cup_{i = 1}^m \cup_{b \in \{0,1\}^m} \calA_i^b$.

\begin{itemize}
	\item Suppose that~$(x_1 + \Delta) \cap \calA = \emptyset$. Then, observe from~\eqref{eq:hypothesis} that~$h(x_1 + \delta) = 1$ for all~$h \in \calH$ and~$\delta \in \Delta$. Hence, for all~$h \in \calH$, we obtain
	\begin{equation*}
		\rhosup_{\delta \sim \fkr}\ \indicator\big[ h(x_1+\delta) \neq y_1 \big] = \indicator\big[ 1 \neq y_1 \big]
			\text{.}
	\end{equation*}
	Hence, depending on the value of~$y_1$, $\calS\big( \{(x_1, y_1), (x_2, y_2)\} \big)$ can either contain sets of the form~$(0,q)$ or~$(1,q)$, for~$q = \{0,1\}$, but not both. For such points, we therefore have~$\abs{\calS\big( \{(x_1, y_1), (x_2, y_2)\} \big)} \leq 2 < 4$. The same argument holds for~$(x_2 + \Delta) \cap \calA = \emptyset$.

	\item Suppose then that both~$(x_1 + \Delta) \cap \calA \neq \emptyset$ and~$(x_2 + \Delta) \cap \calA \neq \emptyset$. Then, $(x_j + \Delta)$ can intersect at most~$m$ sets~$\calA_{i}^b$ with the same signature~$b$~(explicitly, $b = 11\dots1$). But, by construction, $\fkr(\calA_i^b) \leq \rho_o/m$, which implies that~$\fkr(\cup_i \calA_i^{11\dots}) \leq \rho_o$. We then consider the possible labels separately:
	\begin{itemize}
		\item for~$y_j = 1$, we know from~\eqref{eq:hypothesis} that~$\indicator\big[ h(x_j+\delta) \neq 1 \big] = 1$ only when~$x_j+\delta \in \calA_i^b$. But since~$\rho \geq \rho_o$, these sets can be ignored when computing the~$\rhosup$ and we get that
		\begin{equation*}
			\rhosup_{\delta \sim \fkr}\ \indicator\big[ h(x_j+\delta) \neq 1 \big] = 0
				\text{,} \quad \text{for all } h \in \calH
				\text{;}
		\end{equation*}

		\item for~$y_j = 0$, we obtain from~\eqref{eq:hypothesis} that~$\indicator\big[ h_b(x_j+\delta) \neq 0 \big] = 1$ everywhere except when~$x_j+\delta \in \calA_i^b$. Recall that for any signature~$b$, it holds that~$\fkr(\{\delta \mid x_j+\delta \in \bigcup_i \calA_i^b\}) \leq \rho_o$~(possibly with equality if~$b = 11\dots$). Hence, since~$\rho < 1 - \rho_o$, it holds that
		\begin{equation*}
			\rhosup_{\delta \sim \fkr}\ \indicator\big[ h(x_j+\delta) \neq 0 \big] = 1
				\text{,} \quad \text{for all } h \in \calH
				\text{.}
		\end{equation*}
	\end{itemize}
	Since the~$\rhosup$ does not vary over~$\calH$ for either~$x_1$ or~$x_2$, we conclude that~$\abs{\calS\big( \{(x_1, y_1), (x_2, y_2)\} \big)} \leq 2 < 4$, i.e., $\calH$ cannot shatter these points.
\end{itemize}
This implies~$d_\text{VC} \leq 1$ and since~$\abs{\calH} > 1$, we obtain~$d_\text{VC} = 1$. Using Lemma~\ref{L:sample_complexity}, we therefore conclude that~$N = \Theta\big( 1/\epsilon^2 \big)$.

Finally, we consider the case of~\eqref{eq:nom-training} for the nominal loss~$g\big( h, x, y \big) = \indicator\big[ h(x) \neq y \big]$. Once again, we take two points~$(x_1, y_1), (x_2, y_2) \in \calX \times \calY$ and proceed case-by-case.

\begin{itemize}
	\item Suppose, without loss of generality, that~$x_1 \notin \calA$. Then, \eqref{eq:hypothesis} yields~$h(x_1) = 1$ for all~$h \in \calH$ and~$\ell\big( h(x_1), y_1 \big) = \indicator\big[ 1 \neq y_1 \big]$. Depending on the value of~$y_1$, $\calS\big( \{(x_1, y_1), (x_2, y_2)\} \big)$ only has sets of the form~$(0,q)$ or~$(1,q)$, $q = \{0,1\}$, but not both. For such points, $\abs{\calS\big( \{(x_1, y_1), (x_2, y_2)\} \big)} \leq 2 < 4$. The same holds for~$x_2$.

	\item Suppose now that~$x_1 \in \calA_{i}^{b_1}$ and~$x_2 \in \calA_{i}^{b_2}$. Then, $h_b(x_j) = 0$ for~$b = b_j$ or~$h_b(x_j) = 1$ for~$b \neq b_j$. Hence,
	\begin{itemize}
		\item if~$b_1 = b_2$, then~$h_b(x_1) = h_b(x_2)$. Depending on the value of the labels, $\calS\big( \{(x_1, y_1), (x_2, y_2)\} \big)$ only has sets of the form~$(q,q)$ or~$(1-q,q)$, $q = \{0,1\}$, but not both. For such points, we once again have~$\abs{\calS\big( \{(x_1, y_1), (x_2, y_2)\} \big)} \leq 2 < 4$;

		\item if~$b_1 \neq b_2$, then~$h_b(x_1) = 1 - h_b(x_2)$. Depending on the value of the labels, $\calS\big( \{(x_1, y_1), (x_2, y_2)\} \big)$ has, once again, only sets of the form~$(q,q)$ or~$(1-q,q)$, $q = \{0,1\}$, but not both. Hence, $\abs{\calS\big( \{(x_1, y_1), (x_2, y_2)\} \big)} \leq 2 < 4$ for such points;
	\end{itemize}
\end{itemize}

In none of these cases~$\calH$ is able to shatter two points, meaning that~$d_\text{VC} \leq 1$. Since~$\abs{\calH} > 1$, it holds that~$d_\text{VC} = 1$ which is indeed the same value as when~$\rho \geq \rho_o$.
\end{proof}
\newpage
\section{Hyperparameter selection and implementation details} \label{app:training}

In this appendix, we discuss hyperparameter selection and computational details.  All experiments were run across two four-GPU workstations, comprising a total of eight Quadro RTX 5000 GPUs.  Our code is available at: \url{https://github.com/arobey1/advbench}.

\subsection{MNIST}

For the MNIST dataset~\cite{MNISTWebPage}, we used a four-layer CNN architecture with two convolutional layers and two feed-forward layers.  To train these models, we use the Adadelta optimizer~\cite{zeiler2012adadelta} to minimize the cross-entropy loss for 150 epochs with no learning rate day and an initial learning rate of 1.0.  All classifiers were evaluated with a 10-step PGD adversary.  To compute the augmented accuracy, we sampled ten samples from $\fkr$ per data point, and to compute the ProbAcc metric, we sample 100 perturbations per data point.  

\subsection{CIFAR-10 and SVHN}

For CIFAR-10~\cite{krizhevsky2009learning} and SVHN~\cite{netzer2011reading}, we used the ResNet-18 architecture~\cite{he2016deep}.  We trained using SGD and an initial learning rate of 0.01 and a momentum of 0.9.  We also used weight decay with a penalty weight of $3.5 \times 10^{-3}$.  All classifiers were trained for 115 epochs, and we decayed the learning rate by a factor of 10 at epochs 55, 75, and 90.

\subsection{Baseline algorithms}

In the experiments section, we trained a number of baseline algorithms.  In what follows, we list the hyperparameters we used for each of these algorithms:

\begin{itemize}
    \item \textbf{PGD.}  For MNIST, we ran seven steps of gradient ascent at training time with a step size of $\alpha=0.1$.  On CIFAR-10 and SVHN, we ran ten steps of gradient ascent at training time with a step size of $\alpha=2/255$.
    \item \textbf{TRADES.}  We used the same step sizes and number of steps as stated about for PGD.  Following the literature~\cite{zhang2019theoretically}, we used a weight of $\beta=6.0$ for all datasets.
    \item \textbf{MART.}  We used the same step sizes and number of steps as stated about for PGD.  Following the literature~\cite{wang2019improving}, we used a weight of $\lambda=5.0$ for all datasets.
    \item \textbf{DALE.}  We used the same step sizes and number of steps as stated about for PGD.  For all datasets, we used a margin of $\rho=0.1$ (note that this $\rho$ is different from the $\rho$ used in the definition of probabilistically robust learning).  For MNIST, we used a dual step size of $\eta_p=1.0$; for CIFAR-10 and SVHN, we used $\eta_p=0.01$.  For MNIST, we used a temperature of $\sqrt{2\eta T}$ of $10^{-3}$; for CIFAR-10 and SVHN, we used $10^{-5}$.
    \item \textbf{TERM.}  We chose the value of $t$ in~\eqref{eq:term} by cross-validation on the set $\{0.1, 0.5, 1.0, 5.0, 10.0, 50.0\}$.
    \item \textbf{N-HMC}  We chose the value of $q$ in~\eqref{eq:rice} by cross-validation on the set $\{10^1, 10^2, 10^3\}$ (which is the same range considered in the experiments in~\cite{rice2021robustness}).
\end{itemize}

\subsection{Hyperparamters for probabilistically robust learning}

We ran sweeps over a range of hyperparameters for Algorithm~\ref{alg:cvar-sgd}.  By selecting $M$ from $\{1, 2, 5, 10, 20\}$, we found more samples from $\fkr$ engendered higher levels of robustness.  Thus, we use $M=20$ throughout.  We use a step size of $\eta_\alpha=1.0$ throughout.  $T$ was also selected by cross validation from $\{1, 2, 5, 10, 20\}$.  In general, it seemed to be the case that more than 10 steps did not result in significant improvements in robustness.  Finally, in Tables~\ref{tab:cifar-accs}--\ref{tab:mnist-accs}, we selected $\rho$ by cross-validation on $\{0.01, 0.05, 0.1, 0.5, 1.0, 2.0, 3.0, 4.0\}$ (note that in practice, $\rho$ can be chosen to be larger than one).  We found that perhaps surprisingly, larger values of $\rho$ tended to engender higher levels of probabilistic robustness through the metric ProbAcc.  However, this may be due to the instability of training for small values of $\rho$.  In Figure~\ref{fig:acc-vs-rob}, we show the robustness trade-offs for a sweep over different values of $\rho$.

\end{document}